\documentclass[pdflatex,sn-mathphys]{sn-jnl}

\jyear{2023}%

\usepackage[labelformat=simple]{subfig} 
\usepackage{graphicx}

\usepackage{amsfonts}       
\usepackage{amssymb}
\usepackage{amsmath}

\theoremstyle{thmstyleone}%

\newtheorem{thm}{Theorem} 

\newtheorem{lemma}[thm]{Lemma} 
 
\newtheorem{corollary}[thm]{Corollary}

\theoremstyle{thmstyletwo}%

\theoremstyle{thmstylethree}%
\DeclareMathOperator{\sinc}{sinc}

\newsavebox{\wmat}
\savebox{\wmat}{$\mathbf{w}=\begin{bmatrix}
5\\
5
\end{bmatrix}$}

\begin{document}

\title[Nonlinear Neurons with Human-like Apical Dendrite Activations]{Nonlinear Neurons with Human-like Apical Dendrite Activations}


\author[1]{\fnm{Mariana-Iuliana} \sur{Georgescu}}\email{georgescu{\_}lily@yahoo.com}

\author*[1]{\fnm{Radu Tudor} \sur{Ionescu}}\email{raducu.ionescu@gmail.com}

\author[2]{\fnm{Nicolae-C\u{a}t\u{a}lin} \sur{Ristea}}\email{r.catalin196@yahoo.ro}

\author[3]{\fnm{Nicu} \sur{Sebe}}\email{niculae.sebe@unitn.it}

\affil*[1]{\orgdiv{Department of Computer Science}, \orgname{University of Bucharest}, \orgaddress{\city{Bucharest}, \country{Romania}}}

\affil[2]{\orgdiv{Department of Telecommunications}, \orgname{University Politehnica of Bucharest}, \orgaddress{\city{Bucharest}, \country{Romania}}}

\affil[3]{\orgdiv{Department of Information Engineering and Computer Science}, \orgname{University of Trento}, \orgaddress{\city{Trento}, \country{Italy}}}


\abstract{In order to classify linearly non-separable data, neurons are typically organized into multi-layer neural networks that are equipped with at least one hidden layer. Inspired by some recent discoveries in neuroscience, we propose a new model of artificial neuron along with a novel activation function enabling the learning of nonlinear decision boundaries using a single neuron. We show that a standard neuron followed by our novel apical dendrite activation (ADA) can learn the XOR logical function with 100\% accuracy. Furthermore, we conduct experiments on six benchmark data sets from computer vision, signal processing  and natural language processing, i.e.~MOROCO, UTKFace, CREMA-D, Fashion-MNIST, Tiny ImageNet and ImageNet, showing that the ADA and the leaky ADA functions provide superior results to Rectified Linear Units (ReLU), leaky ReLU, RBF and Swish, for various neural network architectures, e.g.~one-hidden-layer or two-hidden-layer multi-layer perceptrons (MLPs) and convolutional neural networks (CNNs) such as LeNet, VGG, ResNet and Character-level CNN. We obtain further performance improvements when we change the standard model of the neuron with our pyramidal neuron with apical dendrite activations (PyNADA). Our code is available at: \url{https://github.com/raduionescu/pynada}.}




\maketitle

\section{Introduction}
\label{sec_intro}

The power of neural networks \citep{Khan-AIR-2020} in classifying linearly non-separable data lies in the use of multiple (at least two) layers. We take inspiration from the recent study of Gidon et al.~\cite{Gidon-S-2020} and propose a simpler yet more effective approach: a new computational model of the neuron, termed \emph{pyramidal neuron with apical dendrite activations} (PyNADA), along with a novel activation function, termed \emph{apical dendrite activation} (ADA), allowing us to classify linearly non-separable data using an individual neuron.

\noindent {\bf Biological motivation.}
Recently, Gidon et al.~\cite{Gidon-S-2020} observed that the apical dendrites of pyramidal neurons in the human cerebral cortex have a different activation function than what was previously known from observations on rodents. The newly-discovered apical dendrite activation function produces maximal amplitudes for electrical currents close to threshold-level stimuli and dampened amplitudes for stronger electrical currents, as shown in Figure~\ref{fig_adaf1}. This new discovery indicates that an individual pyramidal neuron from the human cerebral cortex can classify linearly non-separable data, contrary to the conventional belief that nonlinear problems require multi-layer neural networks. This is the main reason that motivated us to propose ADA and PyNADA.

\begin{figure}[!b]
\begin{center}
\centering
\subfloat[Activation function observed in apical dendrites of pyramidal neurons in the human cerebral cortex.]
{
	\includegraphics[width=0.4\linewidth]{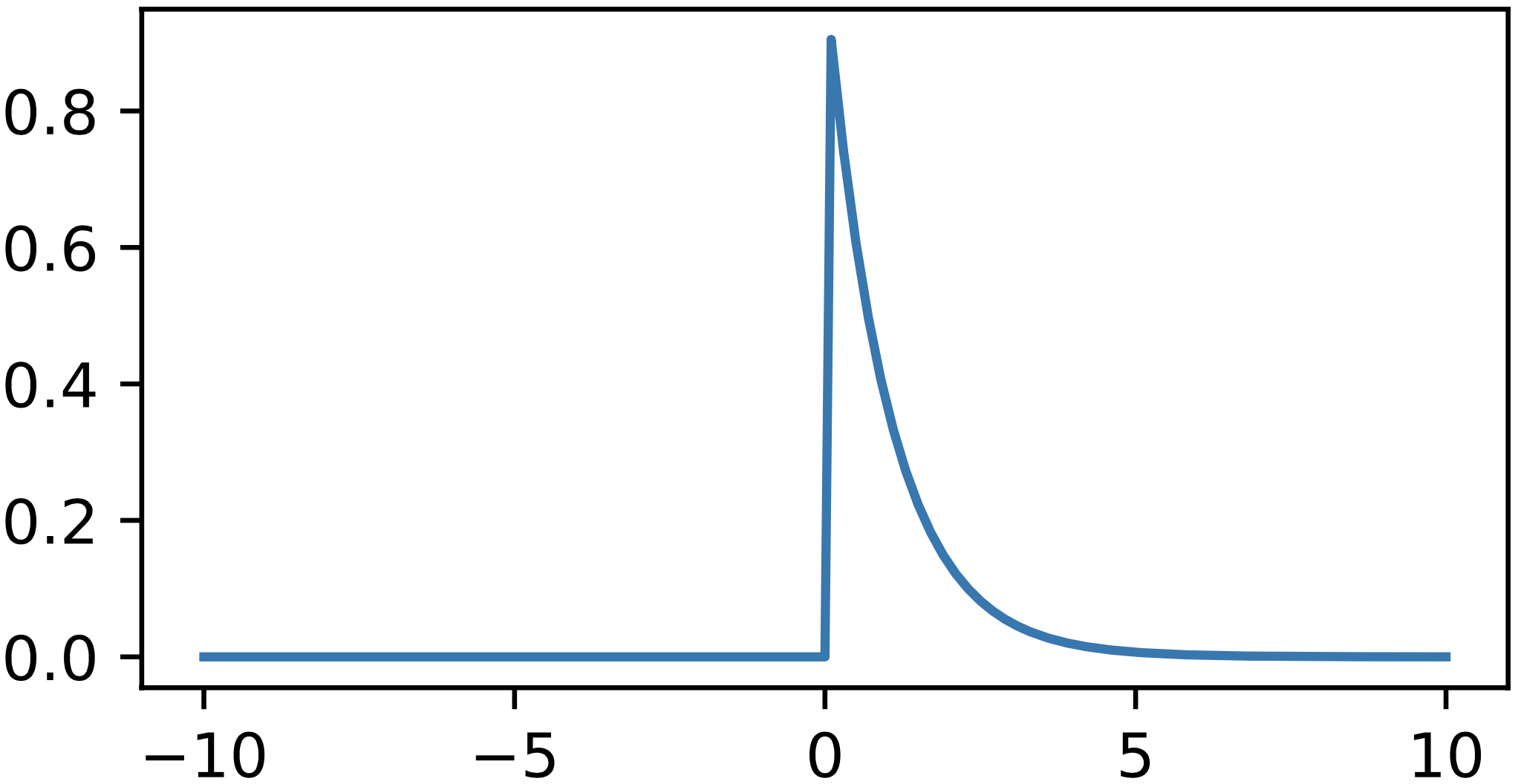}
    \label{fig_adaf1} 
}
\hspace{0.06\linewidth}
\subfloat[Our leaky apical dendrite activation (ADA) function that can be expressed in closed form. This output is obtained by setting $l = 0.005$, $\alpha=1$ and $c=1$ in Equation~\eqref{eq_leaky_adaf3}.]
{
 	\includegraphics[width=0.4\linewidth]{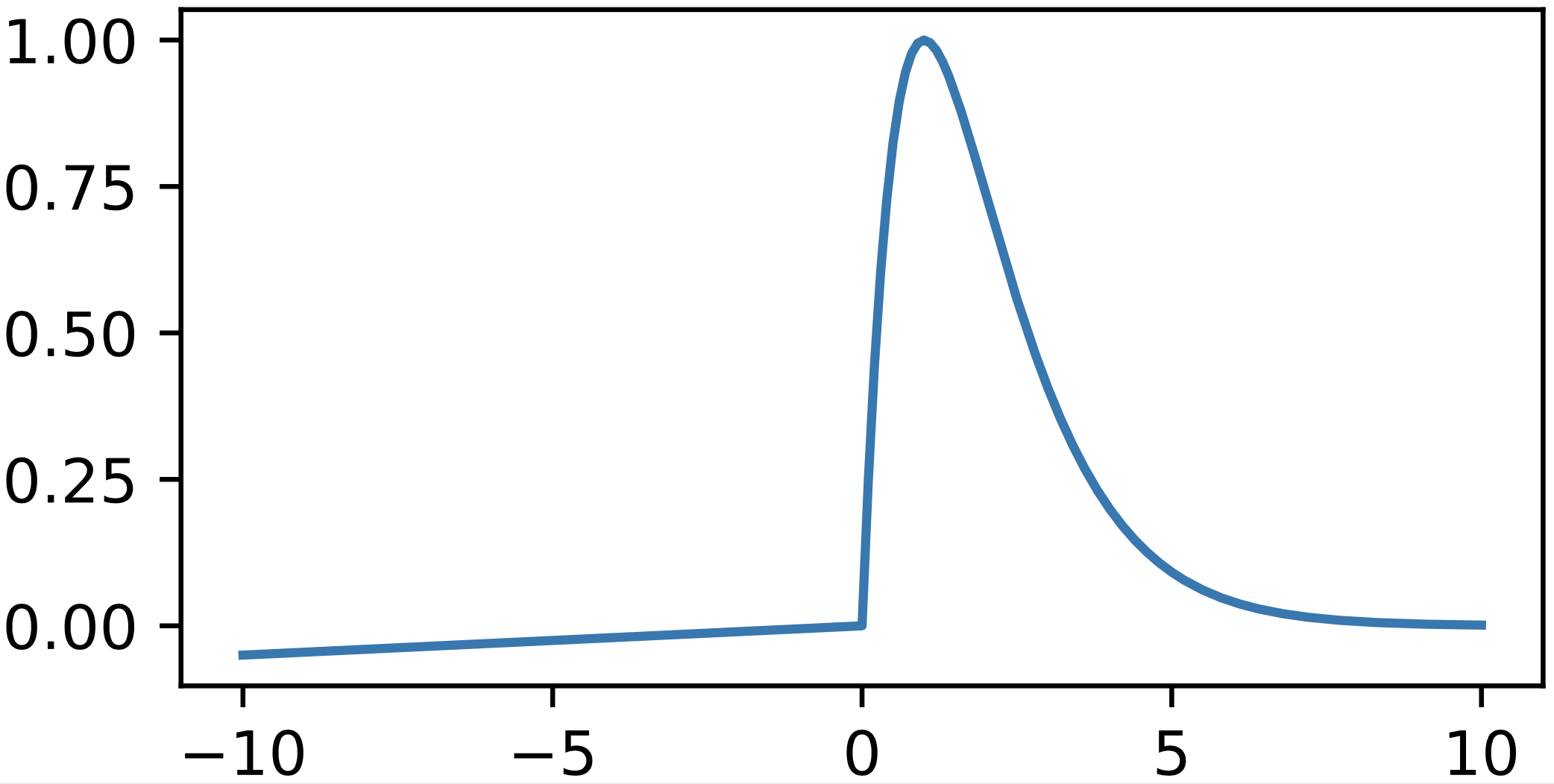}
    \label{fig_leaky_adaf3} 
}
\caption{Original~\citep{Gidon-S-2020} and proposed versions of the apical dendrite activation. The input corresponds to the horizontal axis and the output to the vertical axis.}
\label{fig_both_adaf}
\end{center}
\vskip -0.2in
\end{figure}

\noindent {\bf Psychological motivation.} 
Remember the first time you ate your favorite dish. Was it better than the second or the last time you ate the same dish? According to Knutson et al.~\cite{Knutson-N-2006}, our brains provide higher responses to novel stimuli than to known (repetitive) stimuli. This means that our brains get bored while eating the same dish over and over again, although the dish might be our favorite. If we were to model the brain response over time for a certain stimulus, we would obtain the function illustrated in Figure~\ref{fig_adaf1}. This is yet another reason to propose and experiment with ADA and PyNADA in a computational framework based on deep neural networks, which try to mimic the brain.

\noindent {\bf Mathematical motivation.}
Despite the recent significant advances brought by deep learning \citep{LeCun-N-2015} in various application domains \citep{Xu-CVIU-2017,Wang-N-2018}, state-of-the-art deep neural networks rely on an old and simple mathematical model of the neuron introduced by Rosenblatt \cite{Rosenblatt-PR-1958}. Minsky and Papert \cite{Minsky-MP-2017} argued that a single artificial neuron is incapable of learning nonlinear functions, such as the XOR function. In order to classify linearly non-separable data, standard artificial neurons are typically organized into multi-layer neural networks that are equipped with at least one hidden layer. Contrary to the common belief, we propose an activation function (ADA) that transforms a single artificial neuron into a nonlinear classifier. We also prove that the nonlinear neuron can learn the XOR logical function with 100\% accuracy. Hence, the ADA function can increase the computational power of individual artificial neurons.

\noindent {\bf Empirical motivation.}
We provide empirical evidence in favor of replacing the commonly-used (e.g.~Rectified Liner Units (ReLU) \citep{Nair-ICML-2010} and leaky ReLU \citep{Maas-WDLASL-2013}) or the recently-proposed (e.g.~Swish \citep{Ramachandran-ICLRW-2018}) activation functions \citep{Apicella-NN-2021,Dubey-NC-2021} with the ones proposed in this work, namely ADA and leaky ADA, in various neural network architectures ranging from one-hidden-layer or two-hidden-layer multi-layer perceptrons (MLPs) to convolutional neural networks (CNNs) such as LeNet \citep{LeCun-PI-1998}, VGG \citep{Simonyan-ICLR-2014}, ResNet \citep{He-CVPR-2016} and Character-level CNN \citep{Zhang-NIPS-2015}. We obtain accuracy improvements on several tasks: object class recognition on Fashion-MNIST \citep{Xiao-A-2017}, ImageNet \citep{Russakovsky-IJCV-2015} and Tiny ImageNet \citep{Russakovsky-IJCV-2015}; gender prediction and age estimation on UTKFace \citep{Zhang-CVPR-2017}; voice emotion recognition on CREMA-D \citep{Cao-TAC-2014}; Romanian dialect identification on MOROCO \citep{Butnaru-ACL-2019}. We report further accuracy improvements when the standard artificial neurons are replaced with our pyramidal neurons with apical dendrite activations.

\noindent {\bf Contribution.}
In summary, our contribution is threefold:
\begin{itemize}
\item We propose a new artificial neuron called the pyramidal neuron with apical dendrite activation (PyNADA), along with a new activation function called the apical dendrite activation (ADA).
\item We demonstrate that, due to the novel apical dendrite activation, a single neuron can learn the XOR logical function.
\item We show that the proposed neural building blocks, ADA and PyNADA, provide superior results compared to standard neurons based on the notorious ReLU and leaky ReLU activations, for a broad range of tasks and neural architectures. In most cases, our improvements are statistically significant.
\end{itemize}

\noindent {\bf Organization.} The remainder of this work is organized as follows. In Section \ref{sec_related_art}, we discuss related articles presenting activation functions and artificial neuron models. In Section \ref{sec_method}, we present our novel apical dendrite activation function and pyramidal neuron. In Section \ref{sec_experiments}, we present a comprehensive set of experiments with multiple neural networks on several data sets. Finally, we draw our conclusions and point out future work directions in Section \ref{sec_conclusion}.
 
\section{Related Work}
\label{sec_related_art}

Since our work introduces a novel activation function as well as a novel type of artificial neuron, we consider articles presenting activation functions and artificial neuron models as closely related to our work. Therefore, we discuss related articles on activation functions in Section \ref{sec_related_activations}, and papers presenting models of artificial neurons in Section \ref{sec_related_models}.

\subsection{Activation Functions}
\label{sec_related_activations}

Since activation functions \citep{Apicella-NN-2021,Dubey-NC-2021} have a large impact on the performance of deep neural networks (DNNs), studying and proposing new activation functions is an interesting and important topic \citep{Hayou-ICML-2019}. Nowadays, perhaps the most popular activation function is ReLU \citep{Nair-ICML-2010,Apicella-NN-2021,Dubey-NC-2021}. Formally, ReLU is defined as $max(0, x)$, where $x$ is a scalar input. Because ReLU is linear on the positive side (for $x>0$), its derivative is $1$, so it does not saturate like \emph{sigmoid} and \emph{tanh}. On the negative side of the domain, ReLU is constant, so the gradient is $0$. Hence, a neuron that uses ReLU as activation function cannot update its weights via gradient-based methods on examples for which the neuron is inactive. 
To eliminate the problem caused by inactive neurons with ReLU activation, Maas et al.~\cite{Maas-WDLASL-2013} introduced leaky ReLU. 
The leaky ReLU function is defined as:
\begin{equation}\label{eq_leaky_relu}
\begin{split}
y = \mbox{ReLU}_{\mbox{\scriptsize{leaky}}}(x,l) = l \cdot \min(0, x) + \max(0,x),
\end{split}
\end{equation}
where $l$ is a number between $0$ and $1$ (typically very close to $0$), allowing the gradient to pass even if $x<0$. While in leaky ReLU the leak parameter $l$ is kept fixed, He et al.~\cite{He-ICCV-2015} proposed Parametric Rectified Linear Units (PReLU), in which the leak parameter $l$ is learned by back-propagation.
Different from ReLU, the Exponential Linear Unit (ELU) \citep{Clevert-ICLR-2016} outputs negative values, while still avoiding the vanishing gradient problem on the positive side of the domain. This helps to bring the mean unit activation down to zero, enabling faster convergence times.
Another generalization of ReLU is the Maxout unit \citep{Goodfellow-ICML-2013}, which, instead of applying an element-wise function, divides the input into $k$ groups of values and then outputs the maximum value across all groups.

In contrast to most recent (ReLU, PReLU, ELU, etc.) and historically-motivated (\emph{sign}, \emph{sigmoid}, \emph{tanh}, etc.) activation functions \citep{Apicella-NN-2021,Dubey-NC-2021}, we propose an activation function that transforms a single artificial neuron into a nonlinear classifier. To support our statement, we prove that a neuron followed by our apical dendrite activation function can learn the XOR logical function. We note that there are other activation functions, e.g.~Swish \citep{Ramachandran-ICLRW-2018} and Radial Basis Function (RBF), that generate nonlinear decision boundaries. Different from Swish and RBF, our activation function is supported by recent neuroscience discoveries \citep{Gidon-S-2020}. In addition, our experiments show that ADA and leaky ADA generally provide superior performance levels.  

\subsection{Models of Artificial Neurons}
\label{sec_related_models}

To the best of our knowledge, one of the first mathematical models of the biological neuron is the perceptron \citep{Rosenblatt-PR-1958}. 
The Rosenblatt's perceptron was introduced along with a rule for updating the weights, which converges to a solution only if the data set is linearly separable. Although the perceptron is a simple and old model, it represents the foundation of modern DNNs.
Different from the Rosenblatt's perceptron, the Adaptive Linear Neuron (ADALINE) \citep{Widrow-TR-1960} updates its weights via stochastic gradient descent, back-propagating the error before applying the sign function. ADALINE has the same disadvantage as Rosenblatt's perceptron, namely that it cannot produce nonlinear decision boundaries. 

More recently, researchers proposed the artificial spiking neuron \citep{Anwani-NC-2020,Sarkar-NE-2022,Tavanaei-NN-2019}, a model where the current state of the neuron is determined by the membrane potential, which can raise or decline for a period of time due to electrical impulses. To date, spiking neural network models are computationally expensive, being commonly avoided in real-world applications because of this problem.

In contrast to existing models of artificial neurons \citep{Rosenblatt-PR-1958,Widrow-TR-1960,Anwani-NC-2020,Sarkar-NE-2022,Tavanaei-NN-2019}, we propose an artificial neuron that has two input branches, the basal branch and the apical tuft. The apical tuft is particularly novel because it uses a novel activation function \citep{Gidon-S-2020} that can solve nonlinearly separable problems.

There are a few works that studied various aspects of the modeling of pyramidal neurons, e.g.~segregated dendrites in the context of deep learning \citep{Guerguiev-EL-2017} or memorizing sequences with active dendrites and multiple integration zones \citep{Hawkins-FNC-2016}. Inspired by the recent discovery of Gidon et al.~\cite{Gidon-S-2020}, to the best of our knowledge, we are the first to propose a human-like artificial pyramidal neuron. Different from previous studies, the apical tuft of our pyramidal neuron is equipped with the novel apical dendrite activation suggested by Gidon et al.~\cite{Gidon-S-2020}. Furthermore, we integrate our pyramidal neuron into various deep neural architectures, showing its benefits over standard artificial neurons. We note that Gidon et al.~\cite{Gidon-S-2020} have not presented the pyramidal neuron in a computational scenario. Hence, we are the first to model it computationally. 

\vspace{-0.15cm}
\section{Method}
\label{sec_method}

We would like to emphasize that the proposed apical dendrite activation simulates a particular feature of the human brain, namely a class of calcium-mediated dendritic action potentials observed by analyzing the dendrites of layer 2/3 pyramidal neurons of the human cerebral cortex, as detailed by Gidon et al.~\cite{Gidon-S-2020}. Based on their neuroscientific discovery, we propose the ADA and leaky ADA functions, as well as the pyramidal neuron with apical dendrite activations (PyNADA). All our contributions represent novel building blocks that can be used to construct more powerful neural network architectures.

We next describe the proposed activation functions and pyramidal neuron. Our presentation starts with an introduction of the activation function in Section \ref{sec_method_ADA}, which includes a proof of our statement regarding the capability of solving nonlinear problems, in particular the XOR classification problem, with a single neuron activated by our function. Our presentation continues with a description of the proposed pyramidal neuron in Section \ref{sec_method_PyNADA}. Finally, in Section {\ref{sec_method_arch}}, we provide examples demonstrating how to integrate ADA and PyNADA into a neural network.

\subsection{ADA: Apical Dendrite Activation Function}
\label{sec_method_ADA}

The activation function illustrated in Figure~\ref{fig_adaf1} is deduced from the ex vivo experiments conducted by Gidon et al.~\cite{Gidon-S-2020} on the human cerebral cortex. It can be formally expressed as follows:
\begin{equation}\label{eq_adaf1}
\begin{split}
y = \left\{\begin{array}{ll} 0, & \mbox{if} \; x < 0 \\ \exp(-x), & \mbox{if} \; x \geq 0 \end{array}\right. ,
\end{split}
\end{equation}
where $x \in \mathbb{R}$ is the input of the activation function and $y$ is the output. 

However, the function defined in Equation~\eqref{eq_adaf1} is not directly useful in practice, because commonly-used deep learning frameworks such as TensorFlow \citep{Abadi-OSDI-2016} and PyTorch \citep{Paszke-NeurIPS-2019} do not cope well with functions containing \emph{if} branches. Consequently, we propose a closed form definition that approximates the activation function defined in Equation~\eqref{eq_adaf1}, as follows: 
\begin{equation}\label{eq_adaf3}
\begin{split}
y = \mbox{ADA}(x, \alpha, c) = \max(0, x) \cdot \exp(-x\!\cdot\!\alpha + c),
\end{split}
\end{equation}
where $x \in \mathbb{R}$ is the input of the activation function, $\alpha > 0$ is a parameter that controls the width of the peak, $c > 0$ is a constant that controls the height of the peak and $y$ is the output. The input $x$ of the activation function is the result of summing the weighted inputs of an artificial neuron. In the remainder of this section, we use bold letters to denote vectors and matrices. Given an input $\mathbf{x} \in \mathbb{R}^n$ and the learnable weights $\mathbf{w} \in \mathbb{R}^n$ and $b \in \mathbb{R}$ of an artificial neuron, the input of the ADA function is computed as follows:
\begin{equation}\label{eq_ada_input}
\begin{split}
x=\mathbf{x} \cdot \mathbf{w} + b.
\end{split}
\end{equation}
From Equations \eqref{eq_adaf3} and \eqref{eq_ada_input}, we can obtain the output of the neuron activated by ADA illustrated in Figure {\ref{fig_ADA_neuron}}, as follows:
\begin{equation}\label{eq_adaf_neuron}
\begin{split}
y = \mbox{ADA}(\mathbf{x} \cdot \mathbf{w} + b, \alpha, c).
\end{split}
\end{equation}

\begin{figure}[!t]
\begin{center}
\centerline{\includegraphics[width=0.75\columnwidth]{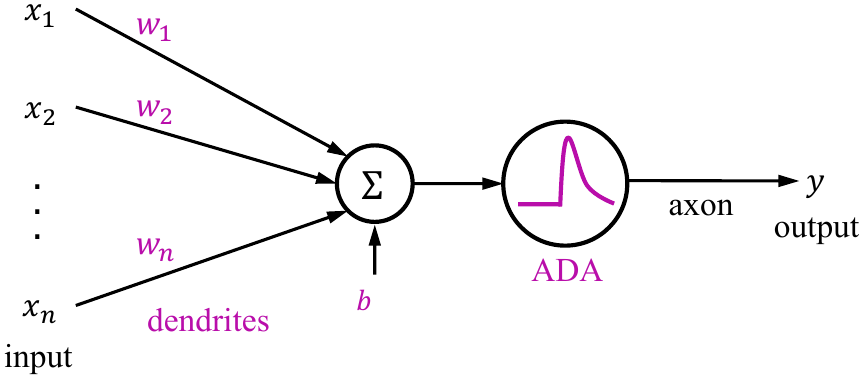}}
\caption{An artificial neuron with apical dendrite activation (ADA). The neuron computes the dot product between the input $\mathbf{x} \in \mathbb{R}^n$ and the weight vector $\mathbf{w} \in \mathbb{R}^n$, and passes the result through the ADA function.}
\label{fig_ADA_neuron}
\end{center}
\end{figure}

Similar to ReLU, our apical dendrite activation (ADA) is saturated on the negative side, i.e.~its gradients are equal to zero for $x < 0$. Thus, a neural model trained with back-propagation \citep{Rumelhart-N-1986} would not update the corresponding weights. We therefore propose leaky ADA, a more generic version that avoids saturation on the negative side, just as leaky ReLU. We formally extend the definition of ADA from Equation~\eqref{eq_adaf3} to leaky ADA as follows:
\begin{equation}\label{eq_leaky_adaf3}
\begin{split}
y\!=\!\mbox{ADA}_{\mbox{\scriptsize{leaky}}}(x, \alpha, c, l) = l\!\cdot\!\min(0,x) + \mbox{ADA}(x, \alpha, c),
\end{split}
\end{equation}
where $0 \leq l \leq 1$ is the leak parameter controlling the function steepness on the negative side and the other parameters are the same as in Equation~\eqref{eq_adaf3}. By setting $l = 0.005$, $\alpha=1$ and $c=1$ in Equation~\eqref{eq_leaky_adaf3}, we obtain the activation function illustrated in Figure~\ref{fig_leaky_adaf3}. By comparing Figure~\ref{fig_adaf1} and Figure~\ref{fig_leaky_adaf3}, we observe that the (leaky) ADA function defined in Equation~\eqref{eq_leaky_adaf3} has a similar shape to the transfer function defined in Equation~\eqref{eq_adaf1}. Indeed, both functions have no activation or almost no activation when $x < 0$. Then, there is a high activation peak for small but positive values of $x$. Finally, the activation dampens along the horizontal axis, as $x$ gets larger and larger. In Algorithm \ref{alg_ADA}, we demonstrate how a neuron activated by ADA works. It essentially involves sequentially applying Eq.~\eqref{eq_ada_input} and Eq.~\eqref{eq_adaf3}. Note that a neuron activated by leaky ADA is formalized analogously, by replacing Eq.~\eqref{eq_adaf3} with Eq.~\eqref{eq_leaky_adaf3} in Algorithm \ref{alg_ADA}.

\begin{algorithm}[!t]
\caption{Neuron with Apical Dendrite Activation\label{alg_ADA}}
\small{
\textbf{Input}: 

$\mathbf{x}$ -- input vector.

$\mathbf{w}$ -- neural weights.

$b$ -- neural bias.

$\alpha$ -- parameter that controls the width of the activation peak.

$c$ -- parameter that controls the height of the peak.

\textbf{Computation}:

$x \leftarrow \mathbf{x} \cdot \mathbf{w} + b$ (apply Eq.~\eqref{eq_ada_input})

$y \leftarrow \max(0, x) \cdot \exp(-x\!\cdot\!\alpha + c)$ (apply Eq.~\eqref{eq_adaf3})

\textbf{Output}: 

$y$ -- the output of the neuron.
}
\end{algorithm}

We next demonstrate that the ADA function can solve a nonlinearly separable problem (XOR). Previously, this type of problems were solved by neural networks using two layers. ADA is more powerful than a typical activation function (ReLU, PReLU, ELU, sigmoid, tanh, etc.), enabling a single neuron to solve the XOR nonlinearly separable problem.

\begin{lemma}\label{prop_ADA_learn_XOR}
There exists an artificial neuron followed by the apical dendrite activation from Equation~\eqref{eq_adaf3} which can predict the labels for the XOR logical function, by rounding its output.
\end{lemma}
\begin{proof}
Given an input data sample represented as a row vector $\mathbf{x} \in \mathbb{R}^n$, the output $y \in \mathbb{R}$ of an artificial neuron with ADA is obtained as follows:
\begin{equation}\label{eq_ADA_neuron1}
\begin{split}
y = \mbox{ADA}(\mathbf{x} \cdot \mathbf{w} + b, \alpha, c),
\end{split}
\end{equation}
where $\alpha$ and $c$ are defined as in Equation~\eqref{eq_adaf3}, $\mathbf{w}$ is the column weight vector and $b$ is the bias term. The following equation shows how to obtain the rounded output:
\begin{equation}\label{eq_ADA_neuron1_round}
\begin{split}
y = \left\lfloor \mbox{ADA}(\mathbf{x} \cdot \mathbf{w} + b, \alpha, c) \right\rceil,
\end{split}
\end{equation}
where $\lfloor \cdot \rceil$ is the rounding function. Similarly, we can obtain the rounded outputs for an entire set of data samples represented as row vectors in an input matrix $\mathbf{X}$:
\begin{equation}\label{eq_ADA_neuron2}
\begin{split}
\mathbf{Y} = \lfloor \mbox{ADA}(\mathbf{X} \cdot \mathbf{w} + b, \alpha, c) \rceil.
\end{split}
\end{equation}

Let $\mathbf{X}$ and $\mathbf{T}$ represent the data samples and the targets corresponding to the XOR logical function, i.e.:
\begin{equation}\label{eq_xor_data}
\begin{split}
\mathbf{X}=\begin{bmatrix}
0 & 0\\
0 & 1\\
1 & 0\\
1 & 1
\end{bmatrix}\!,\;
\mathbf{T}=\begin{bmatrix}
0\\
1\\
1\\
0
\end{bmatrix}\!.
\end{split}
\end{equation}

We next provide an example of weights and parameters\vspace{0.1cm} to prove our lemma. By setting $\mathbf{w}=\begin{bmatrix}
5\\
5
\end{bmatrix}\vspace{0.1cm}$, $b=-4$, $\alpha=1$ and $c=1$ in Equation~\eqref{eq_ADA_neuron2}, we obtain the following output:
\begin{equation}\label{eq_proof}
\begin{split}
\!\!\mathbf{Y}\!=\!\left\lfloor \mbox{ADA}\!\left(\!\begin{bmatrix}
0 & 0\\
0 & 1\\
1 & 0\\
1 & 1
\end{bmatrix}\!\cdot\!\begin{bmatrix}
5\\
5
\end{bmatrix}
-4, 1, 1\!\right)\!\right\rceil
\!=\!\left\lfloor \mbox{ADA}\!\left(\!\begin{bmatrix}
-4\\
1\\
1\\
6
\end{bmatrix}\!, 1, 1\!\right)\!\right\rceil\!\approx\!\left\lfloor \begin{bmatrix}
0\\
1\\
1\\
0.04
\end{bmatrix}
\right\rceil\!=\!\begin{bmatrix}
0\\
1\\
1\\
0
\end{bmatrix}\!.
\end{split}
\end{equation}
Since the output $\mathbf{Y}$ computed in Equation~\eqref{eq_proof} is equal to the target $\mathbf{T}$ defined in Equation~\eqref{eq_xor_data}, it results that Lemma~\ref{prop_ADA_learn_XOR} is true.
\end{proof}

\begin{figure}[!t]
\begin{center}
\centerline{\includegraphics[width=0.6\columnwidth]{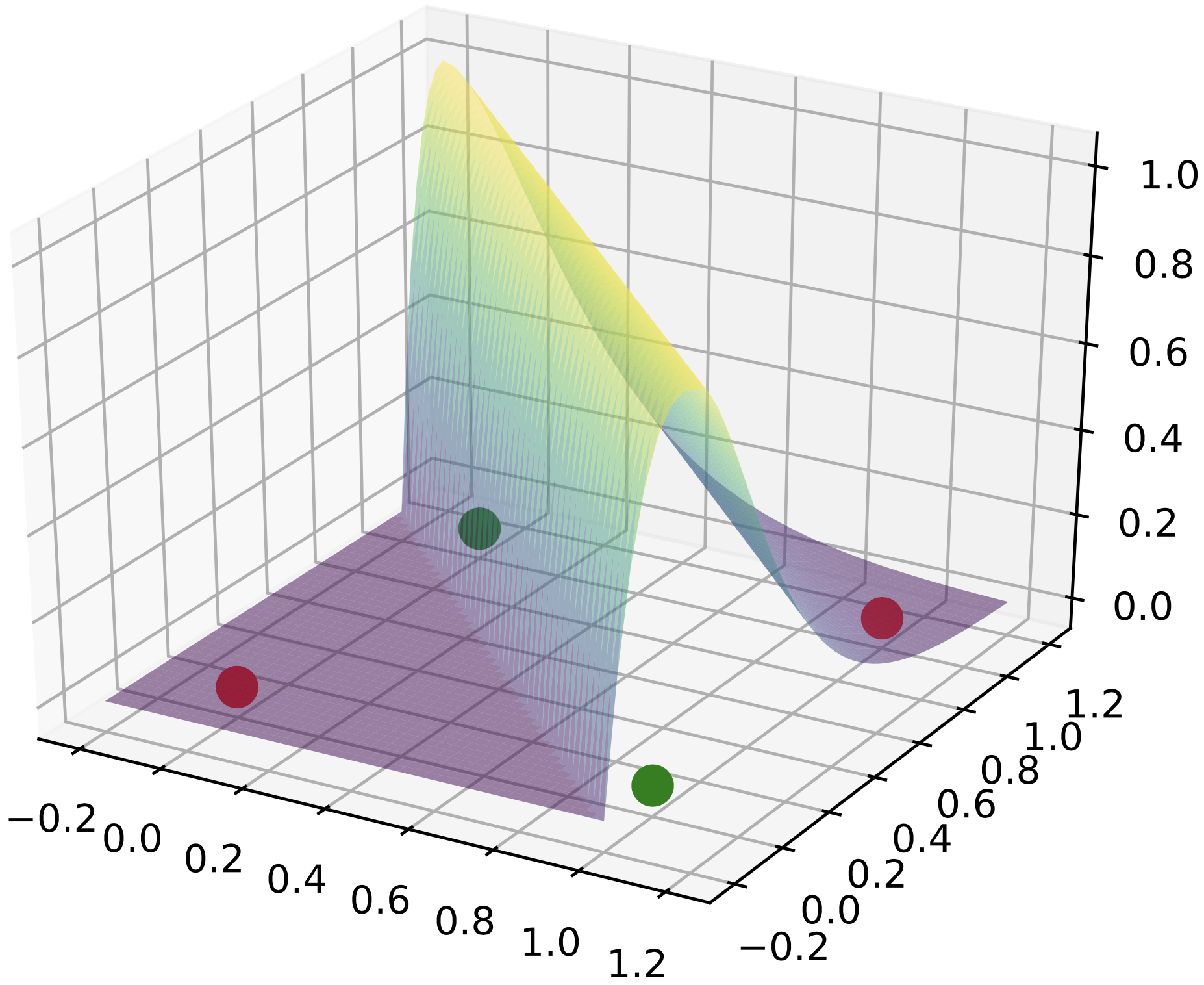}}
\caption{The output of a neuron with apical dendrite activation, as defined in Equation~\eqref{eq_ADA_neuron1}, is able to classify the XOR logical function. The output is obtained by setting the weights to \usebox{\wmat}, the bias term to $b=-4$ and the parameters of the ADA function to $\alpha=1$ and $c=1$. Large output values (closer to $1$) correspond to the green data points labeled as class $1$, while low output values (closer to $0$) correspond to the red data points labeled as class $0$. Best viewed in color.}
\label{fig_xor}
\end{center}
\end{figure}

Our proof is intuitively explained in Figure~\ref{fig_xor}. The four data points from the XOR data set are represented on a plane and the output of the neuron followed by ADA is represented on the axis perpendicular to the plane in which the XOR points reside. The output for the red points (labeled as class $0$) is $0$ or close to $0$, while the output for the green points (labeled as class $1$) is $1$. Applying the rounding function $\lfloor \cdot \rceil$ on top of the output depicted in Figure~\ref{fig_xor} is equivalent to setting a threshold equal to $0.5$, labeling all points above the threshold with class $1$ and all points below the threshold with class $0$. This gives us the labels for the XOR logical function.

Not surprisingly, ADA is still able to solve linearly separable functions, such as OR and AND, without requiring any modification to its definition. We demonstrate this below.

\begin{corollary}\label{prop_ADA_learn_OR}
There exists an artificial neuron followed by the apical dendrite activation from Equation~\eqref{eq_adaf3} which can predict the labels for the OR logical function, by rounding its output.
\end{corollary}

\begin{proof}
We can trivially prove Corollary~\ref{prop_ADA_learn_OR} by following the proof for Lemma~\ref{prop_ADA_learn_XOR}. We just have to set $\alpha$ and $c$ to different values, e.g.~$\alpha=0.4$ and $c=0.5$.
\end{proof}

\begin{corollary}\label{prop_ADA_learn_AND}
There exists an artificial neuron followed by the apical dendrite activation from Equation~\eqref{eq_adaf3} which can predict the labels for the AND logical function, by rounding its output.
\end{corollary}

\begin{proof}
We can trivially prove Corollary~\ref{prop_ADA_learn_AND} by following the proof for Lemma~\ref{prop_ADA_learn_XOR}. We just have to set the bias term to a different value, e.g.~$b=-9$.
\end{proof}

From Lemma~\ref{prop_ADA_learn_XOR}, Corollary~\ref{prop_ADA_learn_OR} and Corollary~\ref{prop_ADA_learn_AND}, it results that an artificial neuron followed by ADA has more computational power than a standard artificial neuron followed by sigmoid, ReLU or other commonly-used activation functions. More precisely, the ADA function enables individual artificial neurons to classify both linearly (e.g.~AND, OR) and nonlinearly (e.g.~XOR) separable data.

\subsection{PyNADA: Pyramidal Neurons with Apical Dendrite Activations}
\label{sec_method_PyNADA}

\begin{figure}[!t]
\begin{center}
\centerline{\includegraphics[width=0.75\columnwidth]{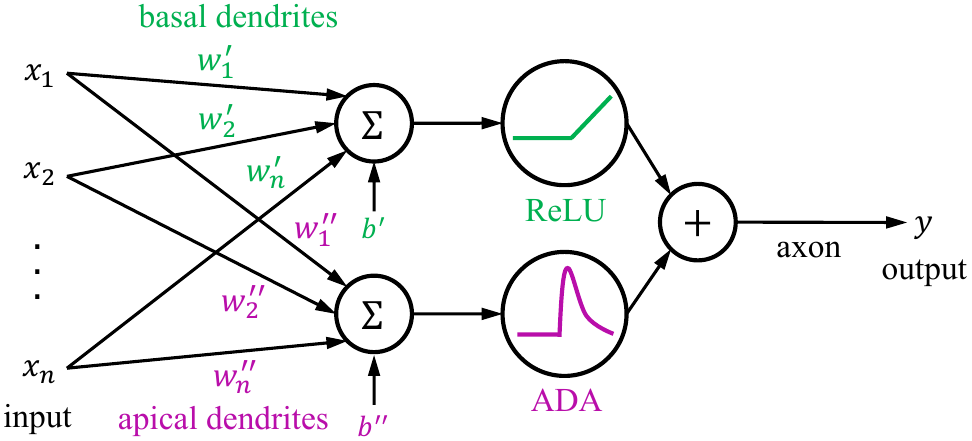}}
\caption{A pyramidal neuron with apical dendrite activations (PyNADA). The input $\mathbf{x}$ goes through the basal dendrites followed by ReLU and through the apical tuft followed by ADA. The results are summed up and passed through the axon. Best viewed in color.}
\label{fig_PyNADA}
\end{center}
\end{figure}

Pyramidal neurons have two types of dendrites: apical dendrites and basal dendrites. Electrical impulses are sent to the neuron through both kinds of dendrites and the impulse is passed down the axon, if an action potential occurs. Prior to Gidon et al.~\cite{Gidon-S-2020}, it was thought that apical and basal dendrites had identical activation functions. This is because experiments were usually conducted on pyramidal neurons extracted from rodents. In this context, proposing an artificial pyramidal neuron would not make much sense, because its mathematical model would be identical to a standard artificial neuron. Gidon et al.~\cite{Gidon-S-2020} observed that the apical dendrites of pyramidal neurons in the human cerebral cortex have a different (previously unknown) activation function, while the basal dendrites exhibit the well-known hard-limit transfer function. This observation calls for a new model of artificial pyramidal neurons. We therefore propose pyramidal neurons with apical dendrite activations (PyNADA).

Given an input data sample $\mathbf{x} \in \mathbb{R}^n$, the output $y \in \mathbb{R}$ of a PyNADA is obtained through the following equation:
\begin{equation}\label{eq_PyNADA}
\begin{split}
\!\!y = \mbox{ReLU}(\mathbf{x} \cdot \mathbf{w}' + b') + \mbox{ADA}(\mathbf{x} \cdot \mathbf{w}'' + b'', \alpha, c),
\end{split}
\end{equation}
where $\alpha$ and $c$ are defined as in Equation~\eqref{eq_adaf3}, $\mathbf{w}'$ and $\mathbf{w}''$ are column weight vectors and $b'$ and $b''$ are bias terms. A graphical representation of PyNADA is provided in Figure~\ref{fig_PyNADA}. In the proposed model, the input $\mathbf{x}$ is distributed to the basal dendrites represented by the weight vector $\mathbf{w}'$ and the bias term $b'$, and to the apical dendrites (apical tuft) represented by the weight vector $\mathbf{w}''$ and the bias term $b''$. Formally, the computation carried out by a pyramidal neuron with apical dendrite activation is detailed in Algorithm \ref{alg_PyNADA}.

For practical reasons, we replace the hard-limit transfer function, suggested by Gidon et al.~\cite{Gidon-S-2020} for the basal dendrites, with the ReLU activation. This change ensures that we can optimize the weights $\mathbf{w}'$ and the bias $b'$ through back-propagation, i.e.~we have at least some non-zero gradients. Since the intensity of electrical impulses is always positive, the biological model proposed by Gidon et al.~\cite{Gidon-S-2020} is defined for positive inputs and the thresholds for the activation functions are well above $0$. However, an artificial neuron can also take as input negative values, i.e.~$\mathbf{x} \in \mathbb{R}^n$. Hence, the thresholds of the activation functions used in PyNADA are set to $0$ (but the bias terms can shift these thresholds).

\begin{algorithm}[!t]
\caption{Pyramidal Neuron with Apical Dendrite Activation\label{alg_PyNADA}}
\small{
\textbf{Input}: 

$\mathbf{x}$ -- input vector.

$\mathbf{w}'$ -- neural weights for basal tuft.

$b'$ -- neural bias for basal tuft.

$\mathbf{w}''$ -- neural weights for apical tuft.

$b''$ -- neural bias for apical tuft.

$\alpha$ -- parameter that controls the width of the activation peak.

$c$ -- parameter that controls the height of the peak.

\textbf{Computation}:

$x' \leftarrow \mathbf{x} \cdot \mathbf{w}' + b'$ (apply Eq.~\eqref{eq_ada_input} for the basal tuft)

$x'' \leftarrow \mathbf{x} \cdot \mathbf{w}'' + b''$ (apply Eq.~\eqref{eq_ada_input} for the apical tuft)

$y \leftarrow \max(0,x') + \max(0, x'') \cdot \exp(-x''\!\cdot\!\alpha + c)$ (apply Eq.~\eqref{eq_PyNADA})

\textbf{Output}: 

$y$ -- the output of the pyramidal neuron.
}
\end{algorithm}

\begin{figure}[!t]
\begin{center}
\centering
\subfloat[A one-hidden-layer MLP with neurons activated by ReLU.]
{
	\includegraphics[width=0.8\linewidth]{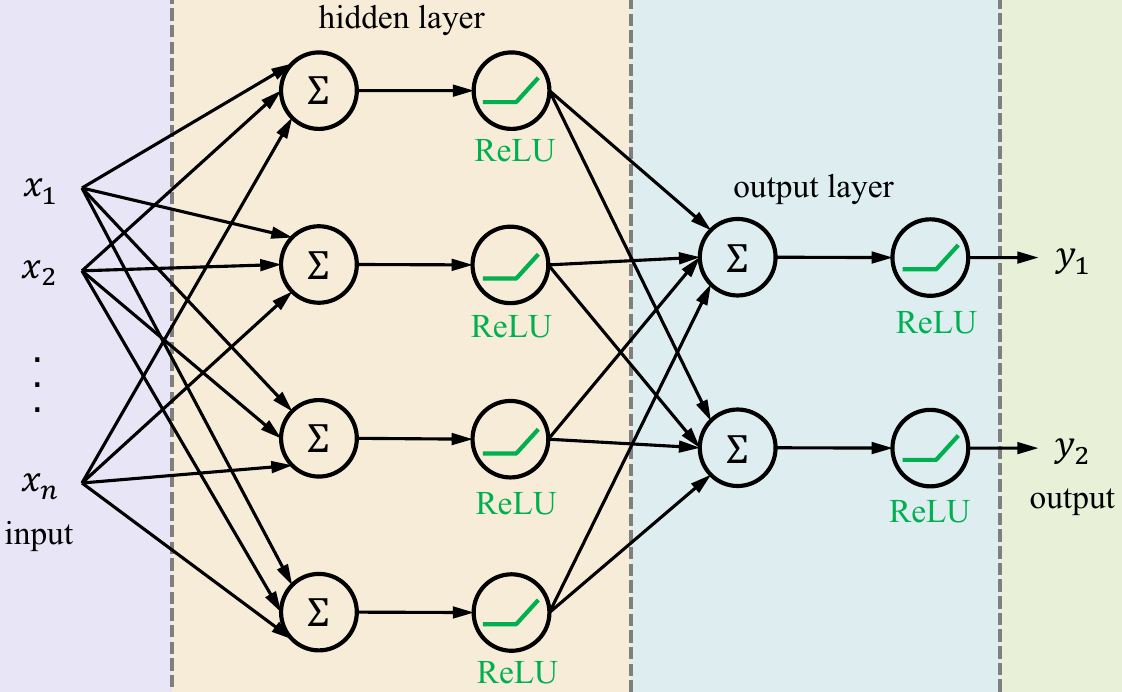}
    \label{fig_mlp_relu} 
}
\hspace{0.06\linewidth}
\subfloat[A one-hidden-layer MLP with neurons activated by ADA.]
{
 	\includegraphics[width=0.8\linewidth]{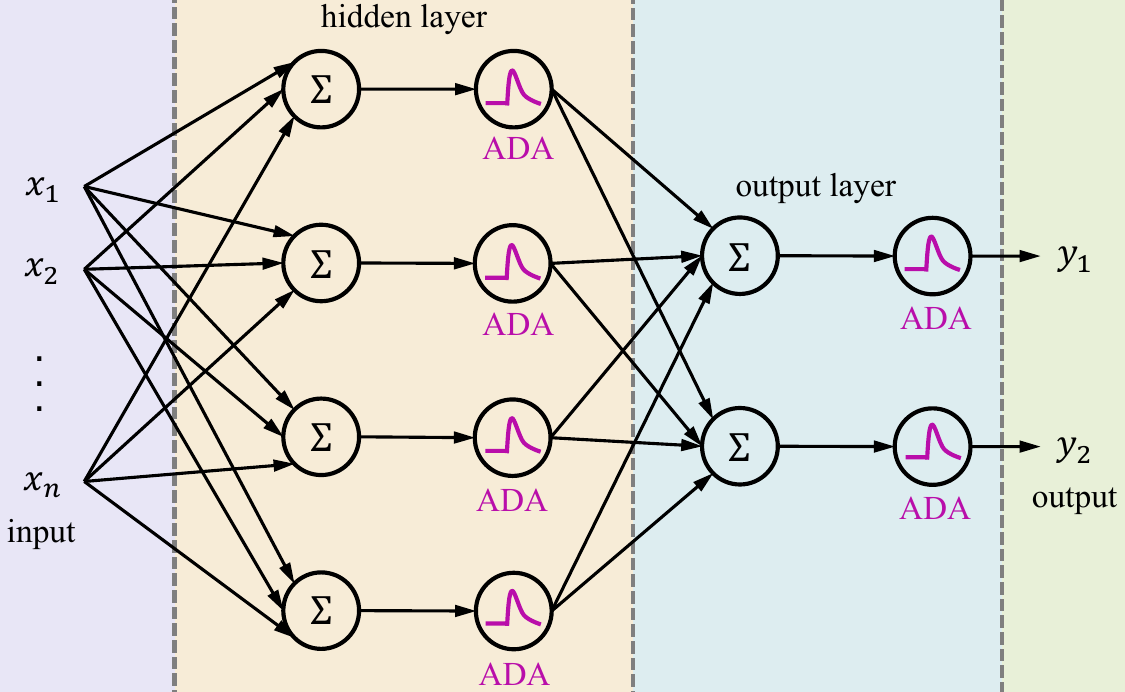}
    \label{fig_mlp_ada} 
}
\caption{An example explaining how ReLU can be replaced with ADA in a standard neural network. Besides changing the activations, the neural architecture remains unchanged. Best viewed in color.}
\label{fig_mlp_relu_vs_ada}
\end{center}
\end{figure}

\begin{figure}[!t]
\begin{center}
\centering
\subfloat[A one-hidden-layer MLP with pyramidal neurons activated by ReLU on both basal and apical tufts, denoted as PyNReLU.]
{
	\includegraphics[width=0.8\linewidth]{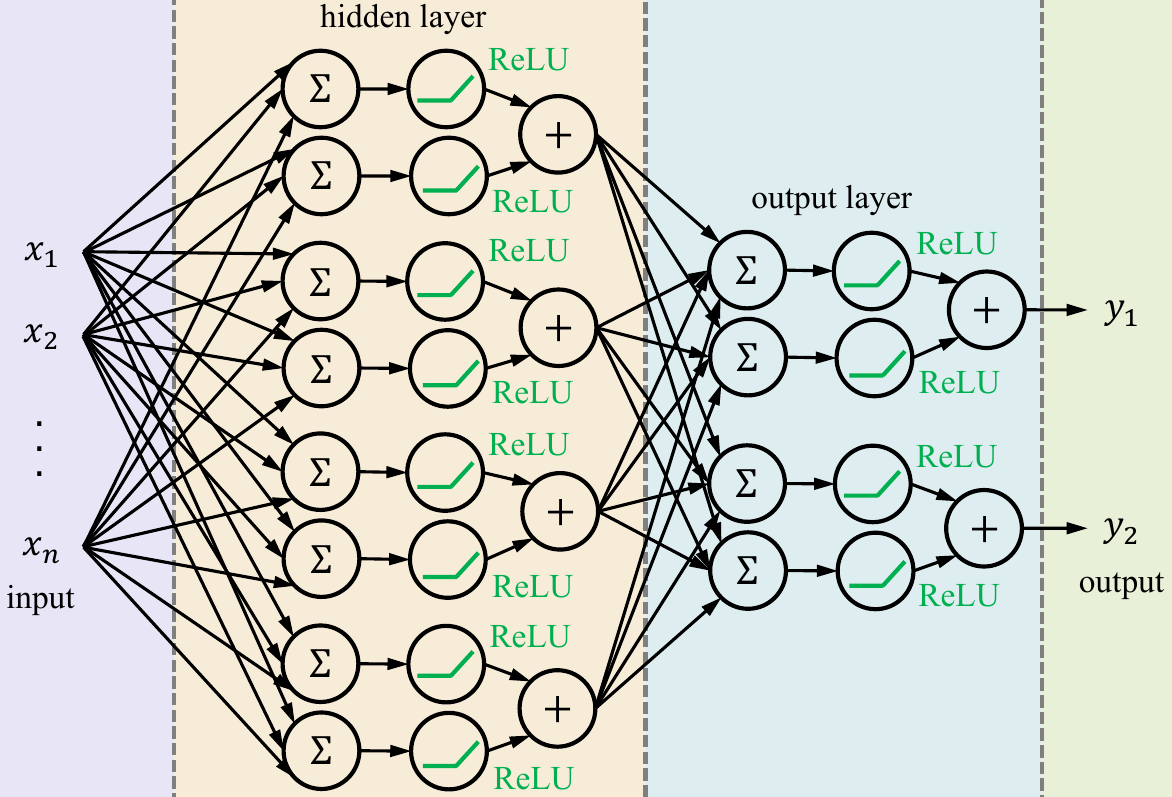}
    \label{fig_mlp_pynrelu} 
}
\hspace{0.06\linewidth}
\subfloat[A one-hidden-layer MLP with pyramidal neurons activated by ReLU and ADA, denoted as PyNADA.]
{
 	\includegraphics[width=0.8\linewidth]{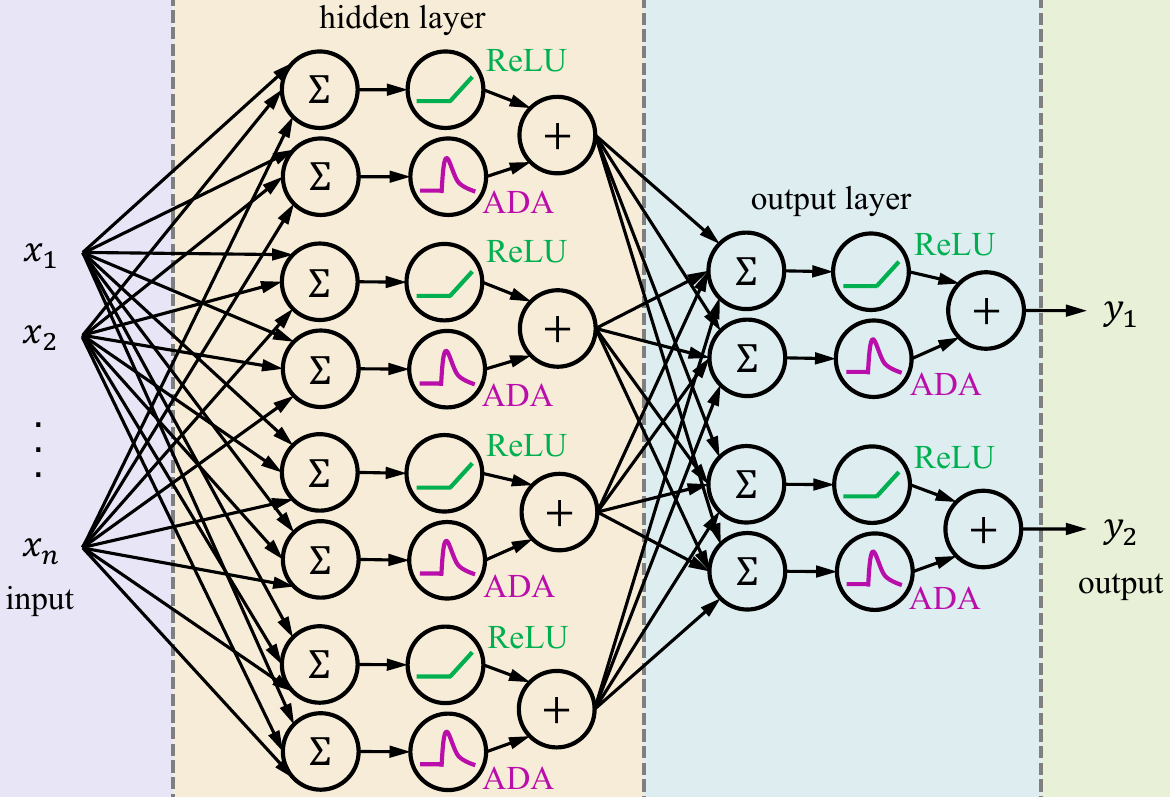}
    \label{fig_mlp_pynada} 
}
\caption{An example explaining how to use PyNReLU and PyNADA units in a standard neural network. Besides switching between ReLU and ADA for the apical tuft, the neural architectures are identical. Best viewed in color.}
\label{fig_mlp_pynrelu_vs_pynada}
\end{center}
\end{figure}

\begin{corollary}\label{prop_PyNADA_learn_XOR}
There exists a pyramidal neuron with apical dendrite activation, as defined in Equation~\eqref{eq_PyNADA}, which can predict the labels for the XOR logical function, by rounding its output.
\end{corollary}

\begin{proof}
We can trivially prove Corollary~\ref{prop_PyNADA_learn_XOR} by following the proof for Lemma~\ref{prop_ADA_learn_XOR}. For the apical tuft, we can simply set\vspace{0.1cm} $\mathbf{w}''=\begin{bmatrix}
5\\
5
\end{bmatrix}\vspace{0.1cm}$, $b''=-4$, $\alpha=1$ and $c=1$ in Equation~\eqref{eq_PyNADA}, just as in the proof for Lemma~\ref{prop_ADA_learn_XOR}. The demonstration results immediately when we simply drop out the basal dendrites\vspace{0.1cm} by setting $\mathbf{w}'=\begin{bmatrix}
0\\
0
\end{bmatrix}$ and $b'=0$.
\end{proof}

By proposing PyNADA, we can jointly leverage the capability of solving nonlinear problems of ADA, and the simplicity and speed of the ReLU activation function, creating more powerful models.

\subsection{Neural Architectures with ADA and PyNADA}
\label{sec_method_arch}

Both ADA and PyNADA are powerful neural building blocks that can increase a network's nonlinear discrimination capacity. Furthermore, they can be easily integrated into any network architecture, without changing the overall structure of the respective neural network. In Figure \ref{fig_mlp_relu_vs_ada}, we illustrate how ADA can replace the ReLU activation in a one-hidden-layer MLP architecture with 4 hidden neurons and 2 output neurons. In Figure \ref{fig_mlp_relu}, we depict the regular architecture based on ReLU activations, while in Figure \ref{fig_mlp_ada}, we illustrate an analogous neural network where ReLU is replaced with ADA. It is fairly easy to observe that replacing ReLU with ADA implies no other changes. 

Similarly, in Figure \ref{fig_mlp_pynrelu_vs_pynada}, we showcase how a one-hidden-layer MLP based on pyramidal neurons should look like. We present two configurations, one where ReLU activations are used for both branches of the pyramidal neurons, called PyNReLU (illustrated in Figure \ref{fig_mlp_pynrelu}), and one where we use ReLU for the basal branch and ADA for the apical branch of each neuron, called PyNADA (illustrated in Figure \ref{fig_mlp_pynada}). Employing pyramidal neurons as in Figure \ref{fig_mlp_pynrelu_vs_pynada} doubles the number of learnable parameters of the network with respect to using standard neurons, as in Figure \ref{fig_mlp_relu_vs_ada}. To this end, we take PyNReLU as the baseline for PyNADA, since both of them share the same number of weights. We once again emphasize that the architectural changes are straightforward and minimal.

Finally, we underline that the architectural changes illustrated in Figure \ref{fig_mlp_relu_vs_ada} and Figure \ref{fig_mlp_pynrelu_vs_pynada} are directly transferable to any type of neural network architecture.

\section{Experiments}
\label{sec_experiments}

To demonstrate the practical abilities of the proposed activation function and artificial neuron, we conduct experiments on six data sets: MOROCO \citep{Butnaru-ACL-2019}, UTKFace \citep{Zhang-CVPR-2017}, CREMA-D \citep{Cao-TAC-2014}, Fashion-MNIST \citep{Xiao-A-2017}, ImageNet \citep{Russakovsky-IJCV-2015} and Tiny ImageNet. In these experiments, we compare our contributions to state-of-the-art activation functions, alternatively inserting the respective activations and our activation into several neural network architectures. In Section \ref{sec_data_sets}, we provide details about the chosen data sets. In Section \ref{sec_eval_setup}, we present the generic evaluation setup used across all six data sets. We discuss the results on the six data sets in Sections \ref{sec_results_moroco}, \ref{sec_results_utk}, \ref{sec_results_crema}, \ref{sec_results_fashion},  \ref{sec_results_tiny}  and \ref{sec_results_imagenet}, respectively. We present a series of ablation results in Section \ref{sec_results_ablation}. We demonstrate the capabilities of approximating nonlinear functions with neural networks based on various activation functions in Section \ref{sec_results_nfa}.
We conclude the experiments with the discussion provided in Section \ref{sec_discussion}.

\subsection{Data Sets}
\label{sec_data_sets}

\noindent 
{\bf MOROCO.} We conduct text classification experiments on MOROCO \citep{Butnaru-ACL-2019}, a data set of 33,564 news articles that are written in Moldavian or Romanian. Each text sample has an average of 309 tokens, the total number of tokens being over 10 millions. The main task is to discriminate between the Moldavian and the Romanian dialects. The data set comes with a training set of 21,719 text samples, a validation set of 5,921 text samples and a test set of 5,924 text samples.

\noindent 
{\bf UTKFace.}
UTKFace \citep{Zhang-CVPR-2017} is a large data set of 23,708 high-resolution images. The images contain faces of people of various ages, genders and ethnic groups. We use the unaligned cropped faces in our experiments. We randomly divide the data set into a training set of 16,595 images ($70\%$), a validation set of 3,556 images ($15\%$) and a test set of 3,557 images ($15\%$). We consider two tasks on UTKFace: gender recognition (binary classification) and age estimation (regression).

\noindent 
{\bf CREMA-D.}
The CREMA-D multimodal database \citep{Cao-TAC-2014} contains 7,442 videoclips of 91 actors (48 male and 43 female) with different ethnic backgrounds. The actors were asked to convey particular emotions while producing, with different intonations, 12 particular sentences that evoke the target emotions. Six labels have been used to discriminate among different emotion categories: neutral, happy, anger, disgust, fear and sad. 
In our experiments, we consider only the audio modality. We split the audio samples into $70\%$ for training, $15\%$ for validation and $15\%$ for testing.

\noindent 
{\bf Fashion-MNIST.}
Fashion-MNIST \citep{Xiao-A-2017} is a recently introduced data set that shares the same structure as the more popular MNIST \citep{LeCun-PI-1998} data set, i.e.~it contains 60,000 training images and 10,000 test images that belong to $10$ classes of fashion items. We use a subset of 10,000 images from the training set for validation. 

\noindent 
{\bf ImageNet.} The ImageNet Large Scale Visual Recognition Challenge (ILSVRC) \citep{Russakovsky-IJCV-2015} is based on a data set of 1000 object categories and about 1.2 million images. We use the official validation set of 50,000 images for comparing the models.

\noindent 
{\bf Tiny ImageNet.}
We present results with longer apical dendrites on Tiny ImageNet, a subset of ImageNet \citep{Russakovsky-IJCV-2015}. The data set provides 500 training images, 50 validation images and 50 test images for 200 object classes. The size of each image is $64 \times 64$ pixels.

\subsection{Evaluation Setup}
\label{sec_eval_setup}

\noindent {\bf Evaluation metrics.}
For the classification tasks (object recognition, gender prediction, voice emotion recognition and dialect identification), we report the classification accuracy. For the regression task (age estimation), we report the mean absolute error (MAE).

\noindent {\bf Baselines.}
We consider several neural architectures ranging from shallow MLPs to deep CNNs: one-hidden-layer MLP, two-hidden-layer MLP, LeNet, VGG-9, ResNet-18, ResNet-50 and character-level CNNs with and without Squeeze-and-Excitation (SE) blocks \citep{Hu-CVPR-2018}. The baseline architectures are based on ReLU, leaky ReLU, RBF or Swish. The latter two activation functions are added because they share one property with ADA, namely the capability to solve XOR when plugged into a single artificial neuron. The goal of our experiments is to study the effect of $(i)$ replacing ReLU and leaky ReLU with ADA and leaky ADA, respectively, and $(ii)$ replacing the standard neurons with our PyNADA. The number of weights in PyNADA is twice as high, but the size of the activation maps is the same as in standard neurons (due to summation of the two branches in PyNADA). For a fair comparison, we included three baseline pyramidal neurons: one with ReLU on both branches (PyNReLU), one with ReLU and RBF (PyNRBF) and one with ReLU and Swish (PyNSwish). All neural models are trained using the Adam optimizer \citep{Kingma-ICLR-2015}. With the exception of ResNet-18 and ResNet-50, which are implemented in PyTorch \citep{Paszke-NeurIPS-2019}, all other neural networks are implemented in TensorFlow \citep{Abadi-OSDI-2016}. Since we employ different architectures in each task, we describe them in more detail in the corresponding subsections below.

In summary, we emphasize that our experiments aim to compare models within the following three groups:
$(i)$ ReLU, RBF, Swish versus ADA; $(ii)$ leaky ReLU versus leaky ADA; $(iii)$ PyNReLU, PyNRBF, PyNSwish versus PyNADA and leaky PyNADA. Note that, for PyNADA, we consider both standard and leaky activations. In leaky PyNADA, the basal dendrites are followed by leaky ReLU and the apical dendrites are followed by leaky ADA. 

Significance testing \cite{Dietterich-NC-1998} is performed within each group, considering the above organization. We also emphasize that the compared models within a group have the same number of parameters and the same computational complexity, e.g.~we do not aim to compare ReLU with PyNADA directly.

We underline that we replace all activations or neurons inside each neural network. For example, when we compare ReLU, RBF, Swish and ADA using a certain neural network, all the activations of the respective network are of the same kind, as illustrated in Figure \ref{fig_mlp_relu_vs_ada}. The same applies to leaky ReLU and leaky ADA. The only case with two types of activation functions in the same network is when we use pyramidal neurons. Each pyramidal neuron has two branches (basal and apical), one with ReLU and the other with one of the following activations: ReLU, RBF, Swish and ADA. When experimenting with pyramidal neurons, we replace all neurons in a neural network with pyramidal neurons, as illustrated in Figure \ref{fig_mlp_pynrelu_vs_pynada}. Moreover, when changing the activation functions or the type of neurons (from conventional to pyramidal), we do not modify the depth of the underlying neural networks.

\noindent {\bf Generic hyperparameter tuning.}
We tune the basic hyperparameters such as the learning rate, the mini-batch size and the number of epochs, for each neural architecture using grid search on the validation set of the corresponding task. We consider learning rates between $10^{-3}$ and $10^{-5}$, mini-batches of $10$, $32$, $64$, $128$ or $800$ samples, and numbers of epochs between $10$ and $300$. Other parameters of the Adam optimizer are used with default values. We use Xavier initialization across all models, regardless of the activation function type or the depth of the network. The parameter tuning is performed using the baseline architectures based on ReLU. Once tuned, the same parameters are used for leaky ReLU, RBF, Swish, ADA, leaky ADA, PyNReLU, PyNRBF, PyNSwish, PyNADA and leaky PyNADA. Hence, we emphasize that the basic parameters are tuned in favor of ReLU. Following Maas et al.~\cite{Maas-WDLASL-2013}, the leak parameter for leaky ReLU and leaky ADA is set to $l=0.01$ in all the experiments, without further tuning.
We note that (leaky) ADA and (leaky) PyNADA have two additional hyperparameters that require tuning on validation. For the constant $c$, we consider two possible values, either $0$ or $1$. For the parameter $\alpha$, we consider two options: $(a)$ perform grid search in the range $[0.1, 1]$ using a step of $0.1$, or $(b)$ learn $\alpha$ using gradient descent during training. As recommended in \citep{Ramachandran-ICLRW-2018}, we use a learnable $\beta$ in Swish. We highlight once again that the same activation function is used in every layer of the entire network, regardless of the depth of the model. Thus, we establish the activation function hyperparameters at the architecture level.

To avoid accuracy variations due to weight initialization or stochastic training, we train each neural model in five consecutive trials, keeping the model with the highest validation performance. In the experiments, we report the performance level of the selected neural models on the held-out test set.

\subsection{Results on MOROCO}
\label{sec_results_moroco}

\begin{table}[t] 
\setlength\tabcolsep{2.5pt}
\caption{Dialect identification accuracy rates (in \%) for two character-level neural models (CNN and CNN+SE) on MOROCO. Results are reported with various activations (ReLU, leaky ReLU, RBF, Swish, ADA, leaky ADA) and artificial neurons (standard, PyNReLU, PyNRBF, PyNSwish and PyNADA). Results that are significantly better than the corresponding baselines, according to a paired McNemar's test \citep{Dietterich-NC-1998}, are marked with $\dagger$ or $\ddagger$ for the significance levels $0.05$ or $0.01$, respectively. Training times are measured on a computer with Nvidia GeForce GTX 1080 GPU with 11GB of RAM. Best model within each group is highlighted in bold.}\label{tab_MOROCO}
\small{
\begin{center}
\begin{tabular}{llccc}
\toprule
\multirow{2}{*}{\bf Model}                           & \multirow{2}{*}{\bf Activation}            & \multirow{2}{*}{\bf Parameter}             & {\bf Test }         & {\bf Time }  \\ 
                          &          &             &    {\bf Accuracy}     & {\bf (s)}\\
\midrule
CNN~\cite{Butnaru-ACL-2019}     & ReLU		            & -	                    & 92.79                 & 27\\
CNN                             & RBF	                & -                     & 54.10                 & 34\\
CNN                             & Swish	                & $\beta=$learnable     & 86.96                 & 33\\
CNN                             & ADA 		            & $\alpha=$learnable    & {\bf 93.68}$^\dagger$       & 33\\
\midrule
CNN                             & leaky ReLU	        & -             	    & 92.90                 & 30\\
CNN                             & leaky ADA 	        & $\alpha=0.4$	        & {\bf 93.55}$^\dagger$       & 37\\
\midrule
CNN+PyNReLU                     & ReLU, ReLU 	        & -	                    & 92.79                 & 52\\
CNN+PyNRBF                      & ReLU, RBF 	        & -	                    & 86.44                 & 58\\
CNN+PyNSwish                    & ReLU, Swish 	        & $\beta=$learnable	    & 91.20                 & 55\\
CNN+PyNADA                      & ReLU, ADA 	        & $\alpha=$learnable	& {\bf 93.61}$^\dagger$       & 52\\
CNN+PyNADA                      & leaky ReLU, leaky ADA & $\alpha=$learnable    & 93.53$^\dagger$       & 57\\
\toprule
CNN+SE~\cite{Butnaru-ACL-2019}  & ReLU		            & -	            	    & 92.99                 & 42\\
CNN+SE                          & RBF	                & -                     & 54.10                 & 48\\
CNN+SE                          & Swish	                & $\beta=$learnable     & 87.99                 & 47\\
CNN+SE                          & ADA 		            & $\alpha=$learnable    & {\bf 93.99}$^\ddagger$      & 47\\
\midrule
CNN+SE                          & leaky ReLU	        & -             	    & 93.06                 & 45\\
CNN+SE                          & leaky ADA 	        & $\alpha=$learnable	& {\bf 93.49}$^\dagger$       & 51\\
\midrule
CNN+SE+PyNReLU                  & ReLU, ReLU 	        & -	                    & 92.97                 & 52\\
CNN+SE+PyNRBF                   & ReLU, RBF 	        & -	                    & 86.23                 & 65\\
CNN+SE+PyNSwish                 & ReLU, Swish 	        & $\beta=$learnable	    & 91.39                 & 61\\
CNN+SE+PyNADA                   & ReLU, ADA 	        & $\alpha=0.5$	        & {\bf 93.72}$^\dagger$       & 63\\
CNN+SE+PyNADA                   & leaky ReLU, leaky ADA & $\alpha=$learnable    & 93.61$^\dagger$       & 60\\
\bottomrule
\end{tabular}
\end{center}
}
\end{table}

\begin{figure}[!t]
\begin{center}
\centerline{\includegraphics[width=0.85\columnwidth]{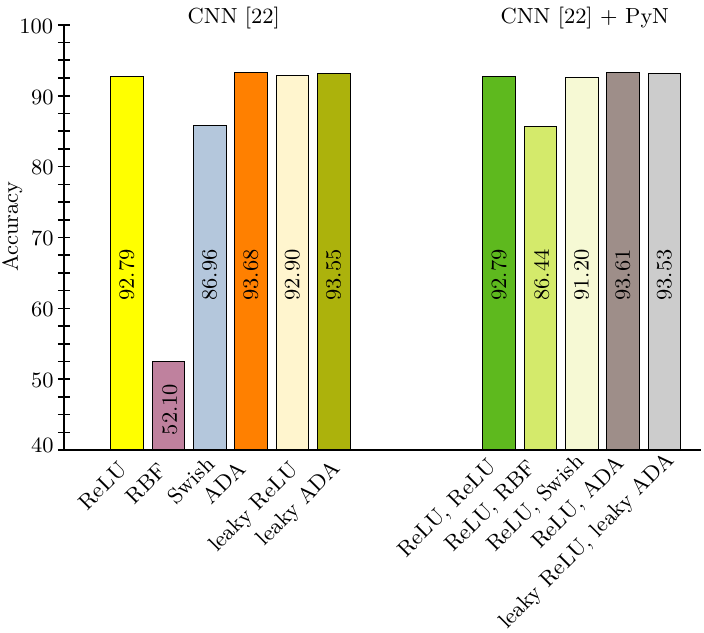}}
\caption{Bar chart of the results obtained on the MOROCO data set by different versions of the character-level CNN \cite{Butnaru-ACL-2019} architecture with both standard and pyramidal neurons activated by various functions: ReLU, RBF, Swish, ADA, leaky ReLU and leaky ADA. Best viewed in color.}
\label{fig_tab_moroco_cnn}
\end{center}
\end{figure}

\begin{figure}[!t]
\begin{center}
\centerline{\includegraphics[width=0.85\columnwidth]{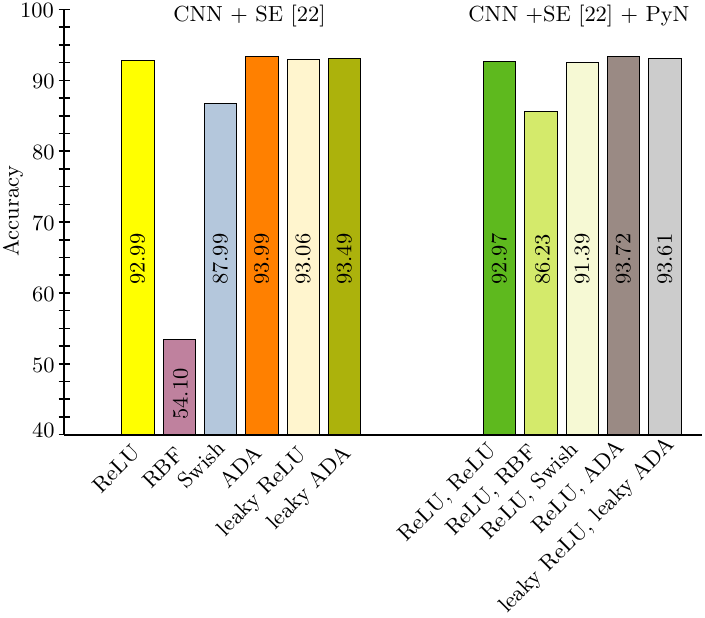}}
\caption{Bar chart of the results obtained on the MOROCO data set by different versions of the character-level CNN+SE \cite{Butnaru-ACL-2019} architecture with both standard and pyramidal neurons activated by various functions: ReLU, RBF, Swish, ADA, leaky ReLU and leaky ADA. Best viewed in color.}
\label{fig_tab_moroco_cnnSE}
\end{center}
\end{figure}

\noindent {\bf Neural architectures.}
For dialect identification, we consider the character-level CNN models presented in \citep{Butnaru-ACL-2019}, which follow closely the model of Zhang et al.~\cite{Zhang-NIPS-2015}. The two CNNs share the same architecture, being composed of an embedding layer, followed by three convolutional and max-pooling blocks, two fully-connected layers with dropout $0.5$, and the final softmax classification layer. The second architecture incorporates an attention mechanism in the form of Squeeze-and-Excitation (SE) blocks \citep{Hu-CVPR-2018} inserted after every convolutional layer. For the SE blocks, we set the reduction ratio to $64$. For both architectures, we keep the same size for the embedding ($256$) and the same number of convolutional filters ($128$) as Butnaru et al.~\cite{Butnaru-ACL-2019}. In fact, our baseline architectures (CNN and CNN+SE) are identical to those of Butnaru et al.~\cite{Butnaru-ACL-2019}.

\noindent {\bf Specific hyperparameter tuning.}
Since Butnaru et al.~\cite{Butnaru-ACL-2019} already tuned the hyperparameters of the character-level CNNs on the MOROCO validation set, we decided to use the same parameters and skip the grid search. Hence, we set the learning rate to $5 \cdot 10^{-4}$ and use mini-batches of $128$ samples. Each CNN is trained for $50$ epochs in 5 trials, keeping the model with the highest validation accuracy for evaluation on the test set. For (leaky) ADA and (leaky) PyNADA, we obtain optimal results with $c=1$, while $\alpha$ is either validated or optimized during training.

\noindent {\bf Results.}
We present the dialect identification results on MOROCO in Table~\ref{tab_MOROCO}. In addition to Table \ref{tab_MOROCO}, we illustrate the results for both standard and pyramidal neurons in Figure \ref{fig_tab_moroco_cnn} (for the CNN \cite{Butnaru-ACL-2019} architecture) and Figure \ref{fig_tab_moroco_cnnSE} (for the CNN+SE {\cite{Butnaru-ACL-2019}} architecture). First, we observe the baseline CNN and CNN+SE models confirm the results reported by Butnaru et al.~\cite{Butnaru-ACL-2019}. 
For the character-level CNN (without SE blocks), we obtain the largest improvement on the test set when we replace ReLU ($92.79\%$) with ADA ($93.68\%$). Furthermore, the improvements of leaky ADA, PyNADA and leaky PyNADA are all higher than $0.6\%$, and the differences are statistically significant. The results are somewhat consistent among the two architectures, CNN and CNN+SE. For example, for the CNN+SE model, we report the largest improvement by replacing ReLU ($92.99\%$) with ADA ($93.99\%$), just as for the CNN without SE blocks. Our highest absolute gain on MOROCO is $1\%$. Overall, the results indicate that all variants of (leaky) ADA and (leaky) PyNADA obtain significantly better results than the corresponding baselines. We observe that neither character-level CNN model is able to converge using RBF as activation. For RBF, we reported the test accuracy corresponding to the last model before the gradients explode. Interestingly, PyNRBF is able to converge, probably due to the basal tuft based on ReLU, but its performance level is not adequate. The networks converge with Swish and PyNSwish, but the corresponding results are much lower than those with (leaky) ReLU or (leaky) ADA. 

\noindent {\bf Running times.}
Since ADA needs to compute the exponential function, it is more computational intensive than ReLU. Moreover, PyNADA has twice more weights than a standard neuron. Hence, in addition to the accuracy rates, we hereby report the training time (in seconds per epoch) in Table~\ref{tab_MOROCO}. With respect to (leaky) ReLU, it seems that (leaky) ADA requires $5$ to $7$ additional seconds per epoch, which means that the training times increases by $15\%$ or $20\%$. The other activation functions that include the exponential function, e.g.~RBF and Swish, have the same disadvantage as ADA.

Meanwhile, (leaky) PyNADA seems to need about $25$ to $30$ extra seconds compared to (leaky) ReLU, increasing the training time by $60\%$ to $100\%$. We thus conclude that the accuracy improvements brought by (leaky) ADA and (leaky) PyNADA come with a non-negligible computational cost with respect to ReLU or leaky ReLU. At the same time, RBF and Swish are slower than ReLU, while also yielding inferior accuracy levels.

As the time measurements are consistent across all benchmarks, we refrain from reporting and commenting on the running times in the subsequent experiments.

\subsection{Results on UTKFace}
\label{sec_results_utk}

\begin{table}[!t] 
\setlength\tabcolsep{2.6pt}
\caption{Gender prediction accuracy rates (in \%) and age estimation MAEs for ResNet-50 on UTKFace. Results are reported with various activations (ReLU, leaky ReLU, RBF, Swish, ADA, leaky ADA) and artificial neurons (standard, PyNReLU, PyNRBF, PyNSwish and PyNADA). Results that are significantly better than the corresponding baselines, according to a paired McNemar's test \citep{Dietterich-NC-1998}, are marked with $\ddagger$ for the significance level $0.01$. Best model within each group is highlighted in bold.}
\label{tab_UTKFace}
\small{
\begin{center}
\begin{tabular}{cllcc}
\toprule
\multirow{2}{*}{\bf Task}          & \multirow{2}{*}{\bf Model}               & \multirow{2}{*}{\bf Activation}  & \multirow{2}{*}{\bf Parameter}       & {\bf Test}\\
          &                &   &        & {\bf Accuracy}\\
\midrule
\multirow{12}{*}{\rotatebox{90}{Gender prediction}} & ResNet-50                    & ReLU		    & -	                    & {\bf 88.56}\\
                    & ResNet-50                    & RBF	        & -                     & 50.12\\
                    & ResNet-50                    & Swish          & $\beta=$learnable     & 87.97\\
                    & ResNet-50                    & ADA 		    & $\alpha=0.1$          & 88.49\\
\cmidrule{2-5}
                    & ResNet-50                    & leaky ReLU     & -                     & 88.91\\
                    & ResNet-50                    & leaky ADA 	    & $\alpha=0.1$	        & {\bf 89.12}\\
\cmidrule{2-5}
                    & ResNet-50+PyNReLU            & ReLU, ReLU 	& -	                    & 88.50\\
                    & ResNet-50+PyNRBF             & ReLU, RBF 	    & -	                    & 90.77\\
                    & ResNet-50+PyNSwish           & ReLU, Swish 	& $\beta=$learnable     & 50.44\\
                    & ResNet-50+PyNADA              & ReLU, ADA             & $\alpha=0.5$  & {\bf 91.65}$^\ddagger$\\
                    & ResNet-50+PyNADA              & leaky ReLU, leaky ADA & $\alpha=0.5$  & 91.53$^\ddagger$\\
\toprule
\multirow{2}{*}{\bf Task}          & \multirow{2}{*}{\bf Model}               & \multirow{2}{*}{\bf Activation}  & \multirow{2}{*}{\bf Parameter}       & {\bf Test}\\
          &                &   &        & {\bf MAE}\\
\midrule
\multirow{12}{*}{\rotatebox{90}{Age estimation}} & ResNet-50     & ReLU           & -	                    & 6.39\\
                    & ResNet-50                    & RBF 		    & -                     & 14.76\\
                    & ResNet-50                    & Swish 		    & $\beta=$learnable     & 6.59\\
                    & ResNet-50                    & ADA 		    & $\alpha=$learnable    & {\bf 5.91}$^\ddagger$\\
\cmidrule{2-5}                    
                    & ResNet-50                    & leaky ReLU	    & -                     & 6.01\\
                    & ResNet-50                    & leaky ADA 	    & $\alpha=$learnable    & {\bf 5.88}$^\ddagger$\\
\cmidrule{2-5} 
                    & ResNet-50+PyNReLU             & ReLU, ReLU 	& -	                    & 7.35\\
                    & ResNet-50+PyNRBF              & ReLU, RBF     & -                     & 6.61\\
                    & ResNet-50+PyNSwish            & ReLU, Swish   & $\beta=$learnable     & 15.12\\
                    & ResNet-50+PyNADA              & ReLU, ADA             & $\alpha=0.5$  & {\bf 5.74}$^\ddagger$\\
                    & ResNet-50+PyNADA              & leaky ReLU, leaky ADA & $\alpha=0.5$  & 5.77$^\ddagger$\\

\bottomrule
\end{tabular}
\end{center}
}
\end{table}

\begin{figure}[!t]
\begin{center}
\centerline{\includegraphics[width=0.85\columnwidth]{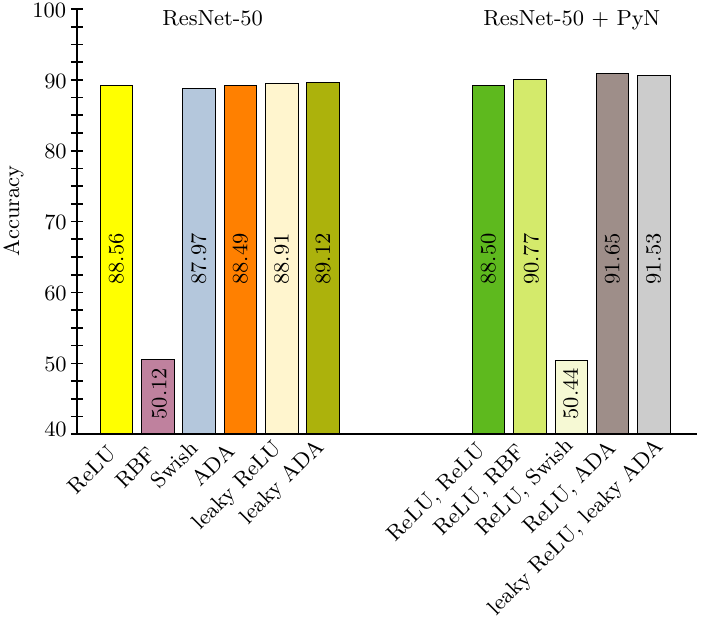}}
\caption{Bar chart of the gender prediction results on the UTKFace data set obtained by different versions of the ResNet-50 architecture with both standard and pyramidal neurons activated by various functions: ReLU, RBF, Swish, ADA, leaky ReLU and leaky ADA. Best viewed in color.}
\label{fig_tab_utk_acc}
\end{center}
\end{figure}

\begin{figure}[!t]
\begin{center}
\centerline{\includegraphics[width=0.85\columnwidth]{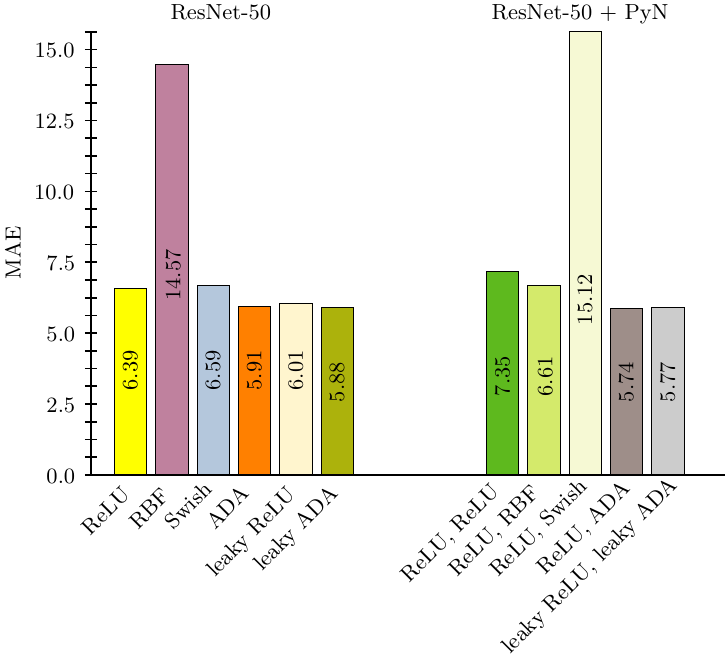}}
\caption{Bar chart of age estimation results on the UTKFace data set obtained by different versions of the ResNet-50 architecture with both standard and pyramidal neurons activated by various functions: ReLU, RBF, Swish, ADA, leaky ReLU and leaky ADA. Best viewed in color.}
\label{fig_tab_utk_mae}
\end{center}
\end{figure}

\noindent {\bf Neural architecture.}
For gender prediction and age estimation, we employ the ResNet-50 architecture \citep{He-CVPR-2016}. Residual networks use batch normalization and skip connections to propagate information over convolutional layers, avoiding the vanishing or exploding gradient problem. This enables effective training of very deep models such as ResNet-50, which is formed of 50 layers. 

\noindent {\bf Specific hyperparameter tuning.} All ResNet-50 variants are trained on mini-batches of $10$ samples using a learning rate of $10^{-4}$. The models are trained for $15$ epochs on the gender prediction task, and for $100$ epochs on the age estimation task. For (leaky) ADA and (leaky) PyNADA, we either validate or learn the parameter $\alpha$, at the same time setting the parameter $c$ to $0$.

\noindent {\bf Results.}
We present the gender prediction and age estimation results in Table~\ref{tab_UTKFace}. In addition, we illustrate the results on gender prediction for both standard and pyramidal neurons in Figure \ref{fig_tab_utk_acc}. In the gender prediction task, we notice that ADA yields slightly lower results than ReLU, while leaky ADA attains slightly better results than leaky ReLU. Both RBF and Swish provide lower results than ReLU and ADA. Nevertheless, we observe significant improvements with PyNADA over PyNReLU. Our highest absolute gain ($3.15\%$) in the gender prediction task is obtained when ResNet-50 is equipped with PyNADA ($91.65\%$) instead of the baseline PyNReLU ($88.50\%$).

Aside from Table \ref{tab_UTKFace}, we also present the results on age estimation for both standard and pyramidal neurons in Figure \ref{fig_tab_utk_mae}. In the age estimation task, we notice that all versions of (leaky) ADA and (leaky) PyNADA surpass the corresponding baselines by significant margins. With an MAE of $5.74$ on the test set, our PyNADA attains the best results in age estimation. With respect to the PyNReLU baseline, PyNADA reduces the average error by $1.61$ years and the difference is statistically significant.

\subsection{Results on CREMA-D}
\label{sec_results_crema}

\begin{table}[!t] 
\caption{Voice emotion recognition accuracy rates (in \%) for ResNet-18 on CREMA-D. Results are reported with various activations (ReLU, leaky ReLU, RBF, Swish, ADA, leaky ADA) and artificial neurons (standard, PyNReLU, PyNRBF, PyNSwish, PyNADA). Results that are significantly better than the corresponding baselines, according to a paired McNemar's test \citep{Dietterich-NC-1998}, are marked with $\ddagger$ for the significance level $0.01$. As reference, the state-of-the-art methods are also included. Best model within each group is highlighted in bold.}
\label{tab_cremaD}
\small{
\begin{center}
\begin{tabular}{llcc}
\toprule
\multirow{2}{*}{\bf Model}              & \multirow{2}{*}{\bf Activation}    & \multirow{2}{*}{\bf Parameter}           & {\bf Test}\\
              &     &          & {\bf Accuracy}\\
\midrule
Shukla et al.~\cite{Shukla-ICASSP-2020}             & ReLU	            & -	            	        & 55.01\\
He et al.~\cite{He-CVPRW-2020}             & leaky ReLU		            & -	            	        & 58.71\\
\midrule
ResNet-18                             & ReLU		        & -	            	        & 62.28\\
ResNet-18                             & RBF	                & -                         & 57.66\\
ResNet-18                             & Swish	            & $\beta=$learnable         & 16.66\\
ResNet-18                             & ADA 		        & $\alpha=0.1$              & {\bf 64.53}$^\ddagger$\\
\midrule
ResNet-18                             & leaky ReLU	        & -             	        & 63.13\\
ResNet-18                             & leaky ADA 	        & $\alpha=0.1$	            & {\bf 65.37}$^\ddagger$\\
\midrule
ResNet-18+PyNReLU                     & ReLU, ReLU 	        & -	                        & 63.39\\
ResNet-18+PyNRBF                      & ReLU, RBF 	        & -	                        & 58.72\\
ResNet-18+PyNSwish                    & ReLU, Swish 	    & $\beta=$learnable         & 16.66\\
ResNet-18+PyNADA                      & ReLU, ADA 	        & $\alpha=0.1$	            & {\bf 65.15}$^\ddagger$\\
ResNet-18+PyNADA                      & leaky ReLU, leaky ADA & $\alpha=0.1$            & 65.10$^\ddagger$\\
\bottomrule
\end{tabular}
\end{center}
}
\end{table}

\begin{figure}[!t]
\begin{center}
\centerline{\includegraphics[width=1.0\columnwidth]{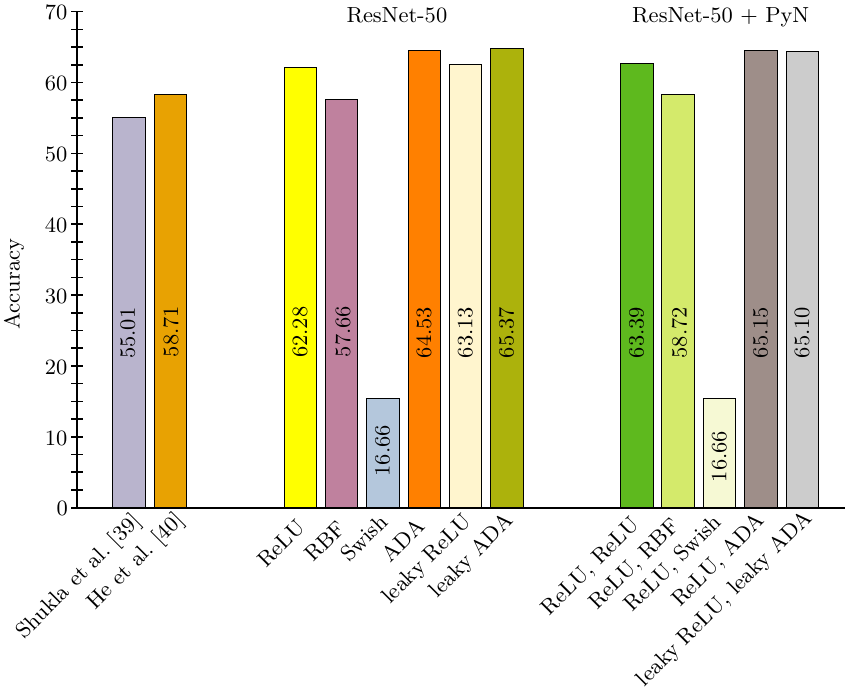}}
\caption{Bar chart of the emotion classification results on the CREMA-D data set obtained by two state-of-the-art methods \citep{Shukla-ICASSP-2020,He-CVPRW-2020}, as well as different versions of the ResNet-18 architecture with both standard and pyramidal neurons activated by various functions: ReLU, RBF, Swish, ADA, leaky ReLU and leaky ADA. Best viewed in color.}
\label{fig_tab_cremad}
\end{center}
\end{figure}

\noindent {\bf Neural architecture.}
For speech emotion recognition, we employ the ResNet-18 architecture \citep{He-CVPR-2016}. 
We modify the number of input channels of ResNet-18 from 3 to 2, in order to feed the network with the Short Time Fourier Transform of the raw audio signals, where the real and the imaginary parts are considered as separate input channels.

\noindent {\bf Specific hyperparameter tuning.} 
All ResNet-18 variants are trained on mini-batches of $16$ samples using a learning rate of $5 \cdot 10^{-4}$. The models are trained for $70$ epochs. For (leaky) ADA and (leaky) PyNADA, we either validate or learn the parameter $\alpha$, at the same time setting the parameter $c$ to $0$. Moreover, we apply training data augmentation using time shifting, thus obtaining more training samples which improve the robustness to data variation of ResNet-18.
 
\noindent {\bf Results.}
We present the speech emotion recognition results in Table~\ref{tab_cremaD}. In addition, we illustrate the results for the emotion recognition task in Figure \ref{fig_tab_cremad}. We observe that ADA and leaky ADA attain superior results compared with the other activation functions, when conventional neurons are employed in ResNet-18. The RBF activation function offers significantly lower results with respect to the most commonly-used activation function, ReLU, while the Swish function fails to converge for both standard and pyramidal neurons. We generally observe that the ResNet-18 based on pyramidal neurons outperforms the ResNet-18 based on conventional neurons, regardless of the activation function. This highlights the effectiveness of the pyramidal design. Our highest absolute gain ($2.24\%$) is obtained for leaky ADA in comparison with leaky ReLU. Moreover, the performance gains brought by (leaky) ADA and (leaky) PyNADA are statistically significant with respect to the corresponding baselines. We notice that our accuracy rates for the audio modality on CREMA-D surpass the  state-of-the-art accuracy levels reported in \citep{Shukla-ICASSP-2020,He-CVPRW-2020}. Given that we report significant improvements over baselines that are already superior to the state of the art, we consider that our results on CREMA-D are remarkable.

\begin{table}[!t]
\setlength\tabcolsep{1.6pt}
\caption{\small{Object class recognition accuracy rates (in \%) for two MLPs (one-hidden-layer MLP and two-hidden-layer MLP) on Fashion-MNIST. Results are reported with various activations (ReLU, leaky ReLU, RBF, Swish, ADA, leaky ADA) and artificial neurons (standard, PyNReLU, PyNRBF, PyNSwish and PyNADA). Results that are significantly better than the corresponding baselines, according to a paired McNemar's test \citep{Dietterich-NC-1998}, are marked with $\dagger$ or $\ddagger$ for the significance levels $0.05$ or $0.01$, respectively. As reference, the results with one-hidden-layer and two-hidden-layer MLPs reported by Xiao et al.~\cite{Xiao-A-2017} are also included. 
Best model within each group is highlighted in bold.}}
\label{tab_Fashion_MNIST_a}
\scriptsize{
\begin{center}
\begin{tabular}{llcc}
\toprule
\multirow{2}{*}{\bf Model}                           & \multirow{2}{*}{\bf Activation}        & \multirow{2}{*}{\bf Parameter}             & {\bf Test}\\
&        &             & {\bf Accuracy}\\
\midrule
one-hidden-layer MLP~\citep{Xiao-A-2017} & ReLU		& -	                    & 87.10\\
one-hidden-layer MLP~\citep{Xiao-A-2017} & tanh		& -	                    & 86.80\\
\midrule
one-hidden-layer MLP              & ReLU		    & -	                    & 88.88\\
one-hidden-layer MLP              & RBF	            & -                     & 88.38\\
one-hidden-layer MLP              & Swish	        & $\beta=$learnable     & 88.85\\
one-hidden-layer MLP              & ADA 		    & $\alpha=0.3$          & {\bf 88.98}\\ 
\midrule
one-hidden-layer MLP              & leaky ReLU	    & -                     & 88.40\\
one-hidden-layer MLP              & leaky ADA 	    & $\alpha=0.3$	        & {\bf 88.97}$^\ddagger$\\ 
\midrule
one-hidden-layer MLP+PyNReLU      & ReLU, ReLU 	    & -                     & 89.00\\
one-hidden-layer MLP+PyNRBF       & ReLU, RBF 	    & -	                    & 88.89\\
one-hidden-layer MLP+PyNSwish     & ReLU, Swish     & $\beta=$learnable     & 89.03\\
one-hidden-layer MLP+PyNADA       & ReLU, ADA 	    & $\alpha=$learnable    & {\bf 89.45}$^\ddagger$\\
one-hidden-layer MLP+PyNADA       & leaky ReLU, leaky ADA   & $\alpha=$learnable & 89.34$^\dagger$\\
\toprule
two-hidden-layer MLP~\citep{Xiao-A-2017} & ReLU		& -	                    & 87.00\\
two-hidden-layer MLP~\citep{Xiao-A-2017} & tanh		& -	                    & 86.30\\
\midrule
two-hidden-layer MLP              & ReLU		    & -	                    & 88.71\\
two-hidden-layer MLP              & RBF       	    & -             	    & 87.84\\
two-hidden-layer MLP              & Swish       	& $\beta=$learnable     & 88.92\\
two-hidden-layer MLP              & ADA 		    & $\alpha=$learnable    & {\bf 88.99}\\
\midrule
two-hidden-layer MLP              & leaky ReLU	    & -                	    & 88.18\\
two-hidden-layer MLP              & leaky ADA 	    & $\alpha=0.1$	        & {\bf 88.93}$^\ddagger$\\
\midrule
two-hidden-layer MLP+PyNReLU      & ReLU, ReLU 	    & -     	    	    & 88.98\\
two-hidden-layer MLP+PyNRBF       & ReLU, RBF 	    & -	                    & 88.43\\
two-hidden-layer MLP+PyNSwish     & ReLU, Swish 	& $\beta=$learnable     & 88.88\\
two-hidden-layer MLP+PyNADA       & ReLU, ADA 	    & $\alpha=$learnable	& 89.40$^\ddagger$\\
two-hidden-layer MLP+PyNADA       & leaky ReLU, leaky ADA   & $\alpha=$learnable & {\bf 89.42}$^\ddagger$\\
\bottomrule
\end{tabular}
\end{center}
}
\end{table}

\begin{figure}[!t]
\begin{center}
\centerline{\includegraphics[width=1.0\columnwidth]{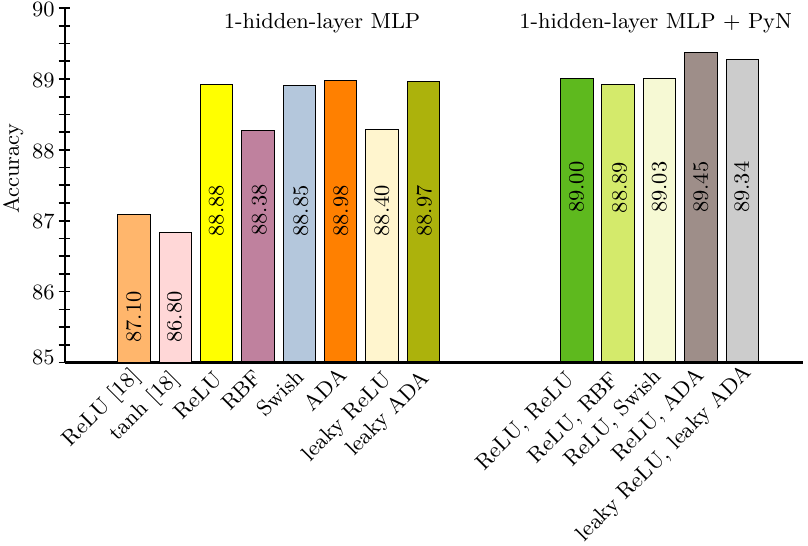}}
\caption{Bar chart of the classification results on the Fashion-MNIST data set obtained by different versions of the one-hidden-layer MLP \cite{Xiao-A-2017} architecture with both standard and pyramidal neurons activated by various functions: ReLU, tanh, RBF, Swish, ADA, leaky ReLU and leaky ADA. Best viewed in color.}
\label{fig_tab_fashion_1h}
\end{center}
\end{figure}

\begin{figure}[!t]
\begin{center}
\centerline{\includegraphics[width=1.0\columnwidth]{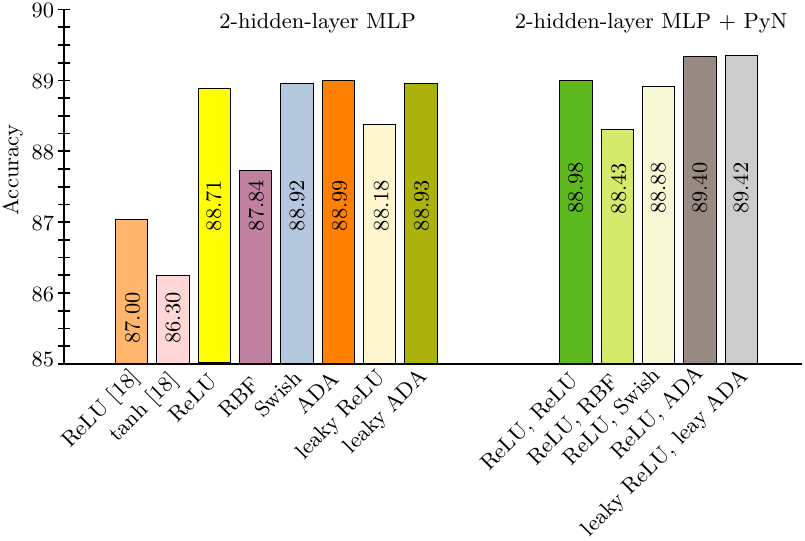}}
\caption{Bar chart of the classification results on the Fashion-MNIST data set obtained by different versions of the two-hidden-layer MLP \cite{Xiao-A-2017} architecture with both standard and pyramidal neurons activated by various functions: ReLU, tanh, RBF, Swish, ADA, leaky ReLU and leaky ADA. Best viewed in color.}
\label{fig_tab_fashion_2h}
\end{center}
\end{figure}

\begin{table}[!t]
\caption{\small{Object class recognition accuracy rates (in \%) for two neural models (LeNet and VGG-9) on Fashion-MNIST. Results are reported with various activations (ReLU, leaky ReLU, RBF, Swish, ADA, leaky ADA) and artificial neurons (standard, PyNReLU, PyNRBF, PyNSwish and PyNADA). Results that are significantly better than the corresponding baselines, according to a paired McNemar's test \citep{Dietterich-NC-1998}, are marked with $\dagger$ or $\ddagger$ for the significance levels $0.05$ or $0.01$, respectively. Best model within each group is highlighted in bold.}}
\label{tab_Fashion_MNIST_b}
\scriptsize{
\begin{center}
\begin{tabular}{llcc}
\toprule
\multirow{2}{*}{\bf Model}                           & \multirow{2}{*}{\bf Activation}        & \multirow{2}{*}{\bf Parameter}             & {\bf Test}\\
&        &             & {\bf Accuracy}\\
\midrule
LeNet                           & ReLU		        & -	                    & 90.84\\
LeNet                           & RBF 		        & -                     & 88.94\\
LeNet                           & Swish             & $\beta=$learnable     & 90.47\\
LeNet                           & ADA 		        & $\alpha=0.3$          & {\bf 91.49}$^\dagger$\\
\midrule
LeNet                           & leaky ReLU	    & -                     & 90.88\\
LeNet                           & leaky ADA 	    & $\alpha=0.9$	        & {\bf 91.38}$^\dagger$\\
\midrule
LeNet+PyNReLU                   & ReLU, ReLU 	    &-   	                & 90.87\\
LeNet+PyNRBF                    & ReLU, RBF 	    & -	                    & 90.43\\
LeNet+PyNSwish                  & ReLU, Swish 	    & $\beta=$learnable     & 91.04\\
LeNet+PyNADA                    & ReLU, ADA 	    & $\alpha=0.3$	        & 91.27\\
LeNet+PyNADA                    & leaky ReLU, leaky ADA & $\alpha=$learnable & {\bf 91.60}$^\ddagger$\\
\toprule
VGG-9                          & ReLU		        & -	            	    & 93.39\\
VGG-9                          & RBF 		        & -              	    & 10.00\\
VGG-9                          & Swish		        & $\beta=$learnable     & 91.96\\
VGG-9                          & ADA 		        & $\alpha=1.0$          & {\bf 93.84}$^\ddagger$\\
\midrule
VGG-9                          & leaky ReLU	    & -             	        & 93.26\\
VGG-9                          & leaky ADA 	    & $\alpha=0.3$	            & {\bf 93.63}$^\dagger$\\
\midrule
VGG-9+PyNReLU                   & ReLU, ReLU 	    & - 	                & 93.49\\
VGG-9+PyNRBF                    & ReLU, RBF 	    & -	                    & 10.00\\
VGG-9+PyNSwish                  & ReLU, Swish 	    & $\beta=$learnable     & 93.70\\
VGG-9+PyNADA                   & ReLU, ADA 	    & $\alpha=1.0$	            & {\bf 93.73}\\
VGG-9+PyNADA                   & leaky ReLU, leaky ADA & $\alpha=1.0$  	    & 93.70\\
\bottomrule
\end{tabular}
\end{center}
}
\end{table}

\begin{figure}[!t]
\begin{center}
\centerline{\includegraphics[width=0.85\columnwidth]{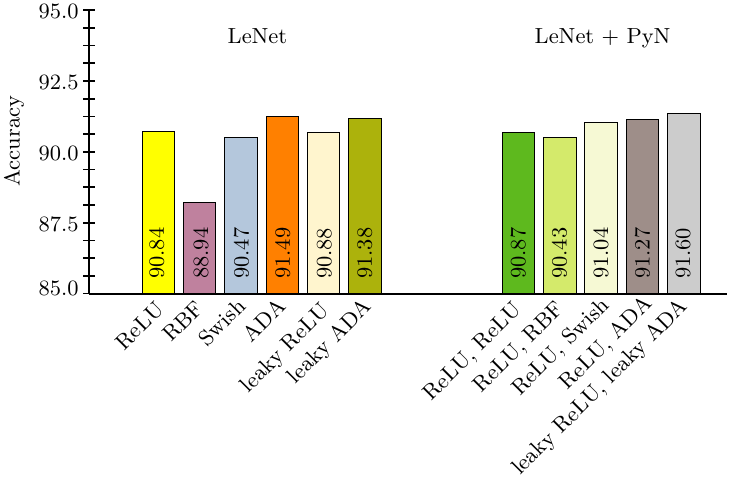}}
\caption{Bar chart of the classification results on the Fashion-MNIST data set obtained by different versions of the LeNet architecture with both standard and pyramidal neurons activated by various functions: ReLU, RBF, Swish, ADA, leaky ReLU and leaky ADA. Best viewed in color.}
\label{fig_tab_fashion_lenet}
\end{center}
\end{figure}

\begin{figure}[!t]
\begin{center}
\centerline{\includegraphics[width=0.85\columnwidth]{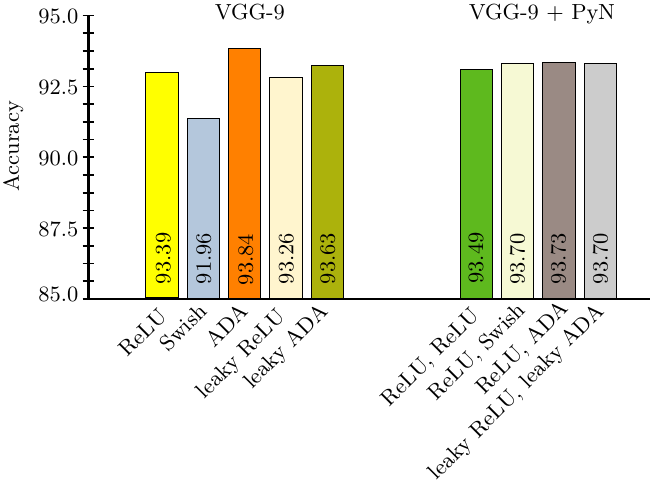}}
\caption{Bar chart of the classification results on the Fashion-MNIST data set obtained by different versions of the VGG-9 architecture with both standard and pyramidal neurons activated by various functions: ReLU, RBF, Swish, ADA, leaky ReLU and leaky ADA. Best viewed in color.}
\label{fig_tab_fashion_vgg}
\end{center}
\end{figure}

\subsection{Results on Fashion-MNIST}
\label{sec_results_fashion}

\noindent {\bf Neural architectures.}
For the Fashion-MNIST data set, we consider two MLPs and two CNNs (LeNet, VGG-9).
We kept the same design choices as in the original paper \citep{Simonyan-ICLR-2014} for VGG-9, but for LeNet \citep{LeCun-PI-1998}, we replaced the average-pooling layer with max-pooling. 
The first MLP architecture (MLP-1) is composed of one hidden layer with 100 units and one output layer with 10 units (the number of classes). The second MLP has two hidden layers with 100 units and 10 units, respectively, followed by the output layer with another 10 units. The considered MLP architectures are similar to those attaining better results among the MLP architectures evaluated by Xiao et al.~\cite{Xiao-A-2017}.

\noindent {\bf Specific hyperparameter tuning.}
We train LeNet for $30$ epochs, using a learning rate of $10^{-3}$ for the first $15$ epochs and $10^{-4}$ for the last $15$ epochs. We train MLP-1 and MLP-2 in the same manner as LeNet. However, VGG-9 is trained for $100$ epochs, starting with a learning rate of $10^{-4}$ in the first $50$ epochs, decreasing it to $10^{-5}$ in the last 50 epochs. We train all models on mini-batches of $64$ images. In all the experiments with (leaky) ADA or (leaky) PyNADA, we either validate or learn $\alpha$, and we set the parameter $c$ to $0$.

\noindent {\bf Results.}
We present the Fashion-MNIST results with various neural architectures and activation functions in Table \ref{tab_Fashion_MNIST_a} and Table \ref{tab_Fashion_MNIST_b}. In addition, we illustrate the results with both standard and pyramidal neurons in Figure \ref{fig_tab_fashion_1h} (for the one-hidden-layer MLP architecture), Figure \ref{fig_tab_fashion_2h} (for the two-hidden-layer MLP architecture), Figure \ref{fig_tab_fashion_lenet} (for the LeNet architecture) and Figure \ref{fig_tab_fashion_vgg} (for the VGG-9 architecture).

Our baseline MLP architectures obtain better accuracy rates than those of Xiao et al.~\cite{Xiao-A-2017}, e.g.~the difference obtained for MLP-1 with ReLU is $1.78\%$ (we report $88.88\%$, while Xiao et al.~\cite{Xiao-A-2017} report $87.10\%$).
We notice that leaky ReLU obtains slightly lower accuracy rates than ReLU, the only exception being the result with LeNet. 
Nonetheless, for both MLP and CNN architectures, the accuracy of ADA on the test set is superior (by up to $0.5\%$) compared to the accuracy of ReLU. We obtain statistically significant improvements for the replacement of ReLU with ADA, from $90.84\%$ to $91.34\%$ using LeNet, and from $93.39\%$ to $93.84\%$ using VGG-9, respectively. Interestingly, we also noticed that the value of the cross-entropy loss is always lower (on both validation and test sets) when we use ADA instead of ReLU. 

We observe that RBF converges only for small networks (MLP-1, MLP-2 or LeNet). When it converges, the results are below the ReLU and leaky ReLU baselines. For VGG-9, the accuracy of RBF is equal to the random choice baseline. Swish attains better results than ReLU for MLP-1 and MLP-2, but below ADA and leaky ADA. For LeNet and VGG-9, Swish surpasses only RBF (all other activation functions are better). When we employ our PyNADA, the accuracy rates improve for all the architectures. The largest improvement of leaky PyNADA on the test set (with respect to the baseline PyNReLU) is $0.73\%$, obtained with LeNet. The reported difference is statistically significant. 


Finally, we emphasize that the results presented in Table \ref{tab_Fashion_MNIST_a}, based on the one-hidden-layer MLP and two-hidden-layer MLP architectures, confirm that the complexity of the network can be reduced by using ADA, leaky ADA, PyNADA or leaky PyNADA. For example, the one-hidden-layer MLP based on ADA ($88.98\%$) outperforms the deeper two-hidden-layer MLPs based on ReLU ($88.71\%$), tanh ($86.30\%$), RBF ($87.84\%$) and Swish ($88.92\%$). Similarly, the one-hidden-layer MLP based on leaky ADA ($88.97\%$) outperforms the deeper two-hidden-layer MLP based on leaky ReLU ($88.18\%$). Moreover, the one-hidden-layer MLPs based on PyNADA ($89.45\%$) and leaky PyNADA ($89.34\%$) outperform the deeper two-hidden-layer MLPs based on PyNReLU ($88.98\%$), PyNRBF ($88.43\%$) and PyNSwish ($88.88\%$).

\subsection{Results on Tiny ImageNet}
\label{sec_results_tiny}

\begin{table*}[!t]
\caption{Object class recognition accuracy rates (in \%) for ResNet-18 on Tiny ImageNet. Results are reported with two artificial neurons, PyNReLU and PyNADA, respectively. Results significantly better than the baseline, according to a paired McNemar's test \citep{Dietterich-NC-1998}, are marked with $\ddagger$ for the significance level $0.01$. Best model is highlighted in bold.}
\label{tab_Tiny}
\small{
\begin{center} 
\begin{tabular}{llcc}
\toprule
{\bf Model}                           & {\bf Activation}        & {\bf Parameter}             & {\bf Test Accuracy}\\
\midrule
ResNet-18+PyNReLU                   & ReLU, ReLU 	 & - 	                & 51.54\\
ResNet-18+PyNRBF                    & ReLU, RBF 	    & -	                    & 51.99\\
ResNet-18+PyNSwish                  & ReLU, Swish 	    & $\beta=$learnable     & 52.93\\
ResNet-18+PyNADA                    & ReLU, ADA 	 & $\alpha=0.1$	        & {\bf 53.86}$^\ddagger$\\ 
\bottomrule
\end{tabular}
\end{center}
}
\end{table*}

We note that, in order to follow exactly the connectivity of pyramidal neurons, the apical dendrites should be longer, i.e.~not connected to neurons in the immediately preceding layer $i\!-\!1$, but to neurons in layers $i\!-\!2$, $i\!-\!3$ or perhaps even further apart. With this design change, the apical dendrites will act as a kind of skip-connections. We hereby show some results in this direction, although this aspect can be studied considering various architecture configurations in future work. Our goal is to prove that the biological design of pyramidal neurons is viable for artificial neural networks as well.

\noindent {\bf Neural architecture.}
For object recognition on Tiny ImageNet, we consider a ResNet-18 architecture \citep{He-CVPR-2016} without standard skip-connections. Instead of standard skip-connections, we use PyNADA with longer apical dendrites, closely modeling the biological pyramidal neurons. Using longer apical dendrites introduces new parameters (e.g., where to connect the apical dendrites) and constraints (the pyramidal design does not apply to every network architecture). The ResNet-18 model allows us to connect the apical dendrites to layer $i\!-\!3$ instead of the immediately preceding layer $i\!-\!1$. As a first baseline for this experiment, we use PyNReLU with equally-long apical dendrites, but with ReLU instead of ADA for the apical tuft. The ResNet-18 with PyNReLU is similar to ResNet-18 with standard skip-connections, the difference being that the skip-connections have learnable weights and ReLU activations. We additionally consider two more baselines with similar design: PyNRBF and PyNSwish.

\noindent {\bf Specific hyperparameter tuning.}
The ResNet-18 models with PyNReLU, PyNRBF, PyNSwish and PyNADA are each trained for $120$ epochs using a learning rate of $10^{-3}$ and mini-batches of $200$ samples. 
For PyNADA, we obtain optimal results with $c=0$ and $\alpha=0.1$ (obtained through validation). For Swish, the parameter $\beta$ is learnable.

\noindent {\bf Results.}
We present the object recognition results on Tiny ImageNet in Table~\ref{tab_Tiny}. First of all, we note that, without pre-training on ImageNet, the accuracy rates on Tiny ImageNet reported in literature are typically around $50\%$ or $60\%$. The baseline PyNReLU attains a fair accuracy of $51.54\%$. PyNADA brings a statistically significant improvement of $2.32\%$ over the baseline. The accuracy rates reached by PyNRBF and PyNSwish are between the accuracy rates of PyNReLU and PyNADA. This experiment shows that it is useful to consider longer apical dendrites in conjunction with ADA, as observed in biology.

\subsection{Results on ImageNet}
\label{sec_results_imagenet}

\begin{table*}[!t]
\caption{Object class recognition accuracy rates (in \%) for top 5 predictions of a ResNet-50 on ImageNet. Results are reported with ReLU and ADA activations and two artificial neurons, PyNReLU and PyNADA, respectively. Results significantly better than the corresponding baseline, according to a paired McNemar's test \citep{Dietterich-NC-1998}, are marked with $\ddagger$ for the significance level $0.01$. Best model within each group is highlighted in bold.}
\label{tab_ImageNet}
\small{
\begin{center} 
\begin{tabular}{llcc}
\toprule
{\bf Model}                           & {\bf Activation}        & {\bf Parameter}             & {\bf Test Accuracy}\\
\midrule
ResNet-50                   & ReLU 	                    & - 	                &  79.31\\
ResNet-50                   & ADA 	                    & $\alpha=0.2$ 	                &  {\bf 79.52}$^\ddagger$ \\
\midrule
ResNet-50+PyNReLU           & ReLU, ReLU 	                & - 	                &  79.88 \\
ResNet-50+PyNADA           & ReLU, ADA 	                & $\alpha=0.2$ 	                & {\bf 80.07}$^\ddagger$ \\

\bottomrule
\end{tabular}
\end{center}
}
\end{table*}

\noindent {\bf Neural architecture.}
For object recognition on ImageNet, we employ the ResNet-50 architecture \citep{He-CVPR-2016}. Since one experiment takes about two weeks on our machine (equipped with two Nvidia GeForce GTX 1080 GPUs with 11GB of RAM), we restrict our comparisons to ReLU versus ADA on the one hand, and PyNReLU versus PyNADA on the other hand. For both PyNReLU and PyNADA, we connect the apical dendrites to layer $i\!-\!3$ instead of the immediately preceding layer $i\!-\!1$, thus following the same design used in the experiments carried out on Tiny ImageNet.

\noindent {\bf Specific hyperparameter tuning.}
The ResNet-50 models with ReLU, ADA, PyNReLU and PyNADA are each trained for $300$ epochs using a learning rate of $10^{-4}$ and mini-batches of $800$ samples. For ADA and PyNADA, we obtain optimal results with $c=0$ and $\alpha=0.2$ (determined through validation).

\noindent {\bf Results.}
We present the object recognition results on ImageNet in Table~\ref{tab_ImageNet}. We observe that the baseline ResNet-50 with ReLU activations attains an accuracy of $79.31\%$, while the same architecture with ADA leads to a higher accuracy of $79.52\%$. The same trend can be observed while comparing the pyramidal architectures, PyNReLU and PyNADA. More precisely, the ResNet-50 based on PyNADA surpasses the baseline ResNet-50 based on PyNReLU by $0.2\%$. These experiments shows once again that ADA and PyNADA outperform the counterparts based on ReLU and PyNReLU, respectively. Furthermore, the ImageNet experiments confirm the Tiny ImageNet experiments, indicating that longer apical dendrites (to layer $i\!-\!3$) in pyramidal neurons are useful.

\subsection{Ablation Results}
\label{sec_results_ablation}

Next, we present an ablation study to assess the influence of the hyperparameter $\alpha$ on the training stability of ADA. We also present experiments on how input noise, labeling noise or the imbalanced nature of the data set affect the performance obtained by ADA and the other competing approaches.

\noindent
{\bf Effect of $\boldsymbol{\alpha}$.} The hyperparameter $\alpha$ defined in Equation~\eqref{eq_adaf3} controls the width of the peak, meaning that it inflates or deflates the non-asymptotic region, which can alleviate the vanishing gradient problem. In order to demonstrate that the ADA function leads to stable training, we perform additional experiments on Fashion-MNIST with the MLP-2 and VGG-9 models, while varying the  parameter $\alpha$. More precisely, we vary $\alpha$ between $0.1$ and $1.0$ with a step of $0.1$, reporting the performance for each value of $\alpha$ in Figure \ref{fashion_mnist_ablation_alpha}. When varying $\alpha$, the performance of the two-hidden-layer MLP fluctuates between $87.3\%$ and $88.9\%$, while the performance of VGG-9 fluctuates between $93.5\%$ and $93.9\%$. These results indicate that the training converges to good optima, regardless of the value of $\alpha$. We thus conclude that these experiments do not show slow convergence, demonstrating that the possibility of encountering vanishing gradients is not to be expected.

\begin{figure} 
\begin{center}
\centerline{\includegraphics[width=1.0\columnwidth]{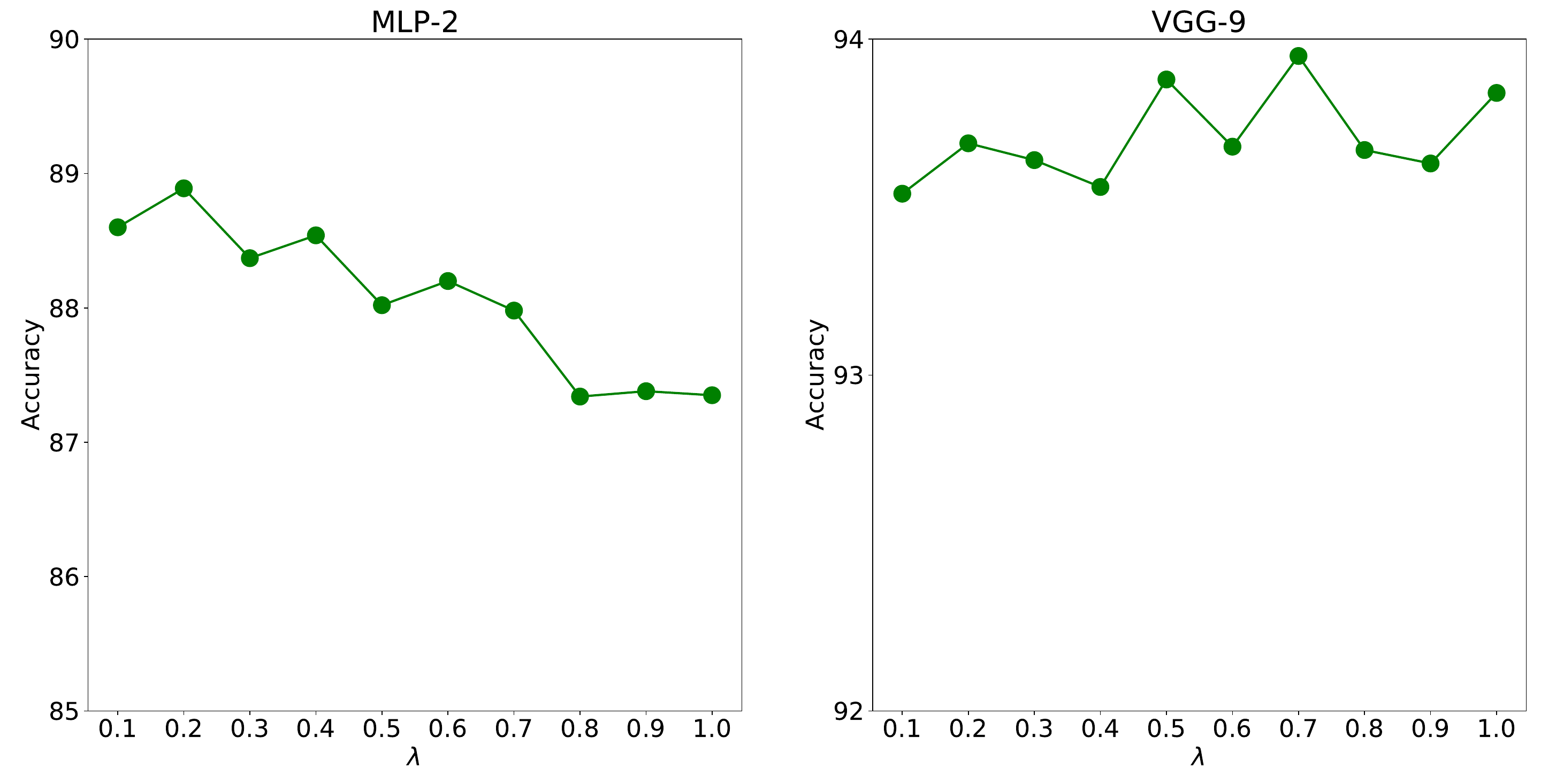}}
\caption{Object class recognition accuracy rates (in \%) for two neural models (MLP-2 and VGG-9) on Fashion-MNIST. Results are reported with the ADA function, while varying the hyperparameter $\alpha$ defined in Equation~\eqref{eq_adaf3}. The training converges for all values of $\alpha$, without showing signs of slow convergence or vanishing gradients.}\label{fashion_mnist_ablation_alpha}
\end{center}
\end{figure}

\begin{figure}  
\begin{center}
\centerline{\includegraphics[width=1.0\columnwidth]{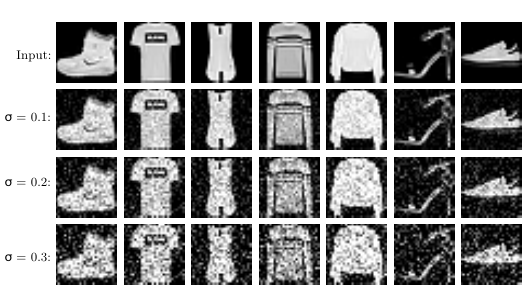}}
\caption{Examples of input images from Fashion-MNIST, with various levels of Gaussian noise. We apply Gaussian noise of $0$ mean and different standard deviations (shown on different rows).}\label{noisy_input_samples_fashionmnist} 
\end{center}
\end{figure}

\begin{figure} 
\begin{center}
\centerline{\includegraphics[width=1.0\columnwidth]{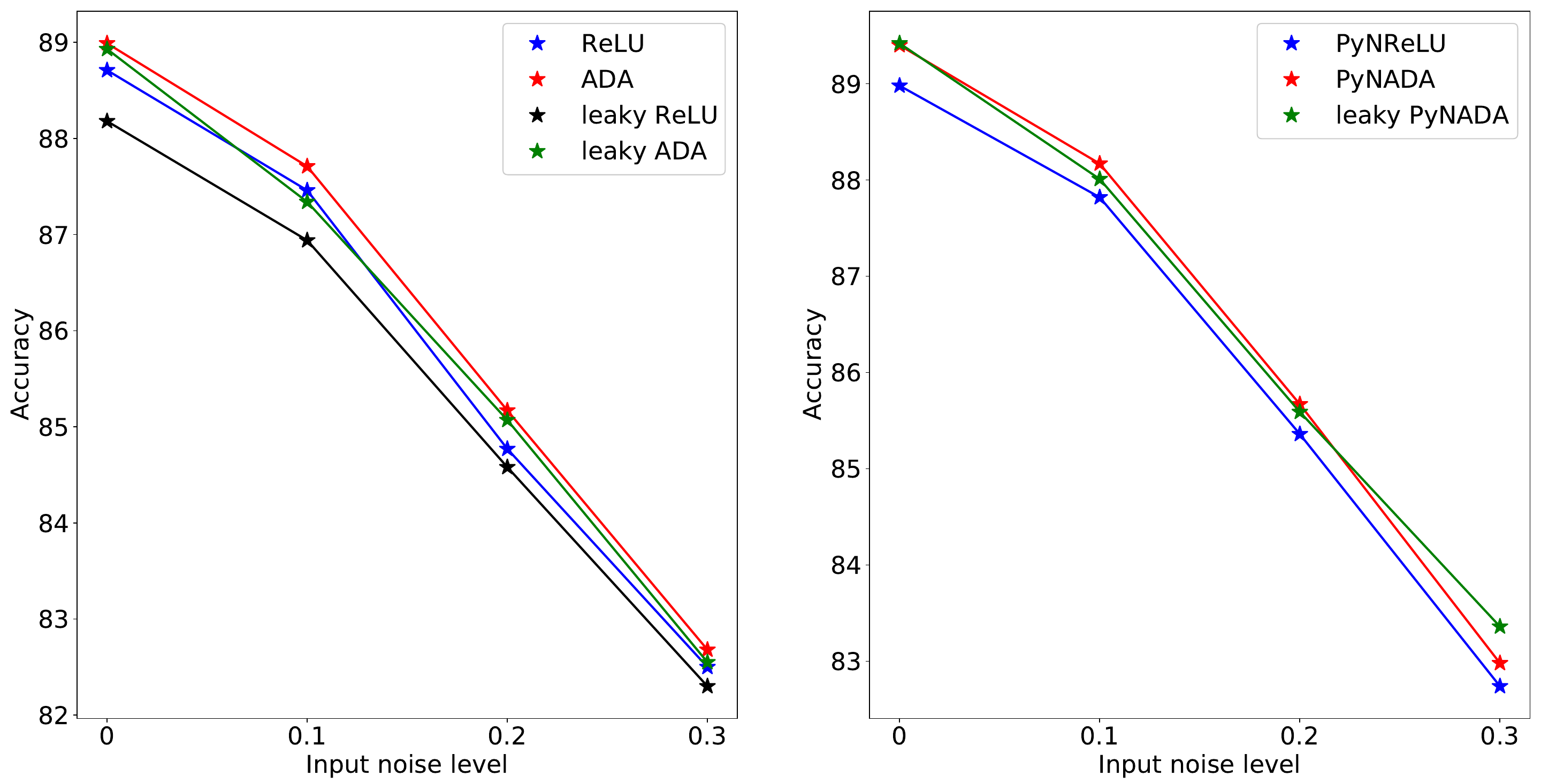}}
\caption{Object class recognition accuracy rates (in \%) for a two-hidden-layer MLP on Fashion-MNIST, with various levels of Gaussian noise applied on the input images. Results are reported with various activations (ReLU, leaky ReLU, ADA, leaky ADA) and artificial neurons (standard, PyNReLU, PyNADA and leaky PyNADA). We add Gaussian noise of $0$ mean and different standard deviations (illustrated on the horizontal axis). Best viewed in color.}\label{noise_input_mlp_fashion_mnist}
\end{center}
\end{figure}

\begin{figure} 
\begin{center}
\centerline{\includegraphics[width=1.0\columnwidth]{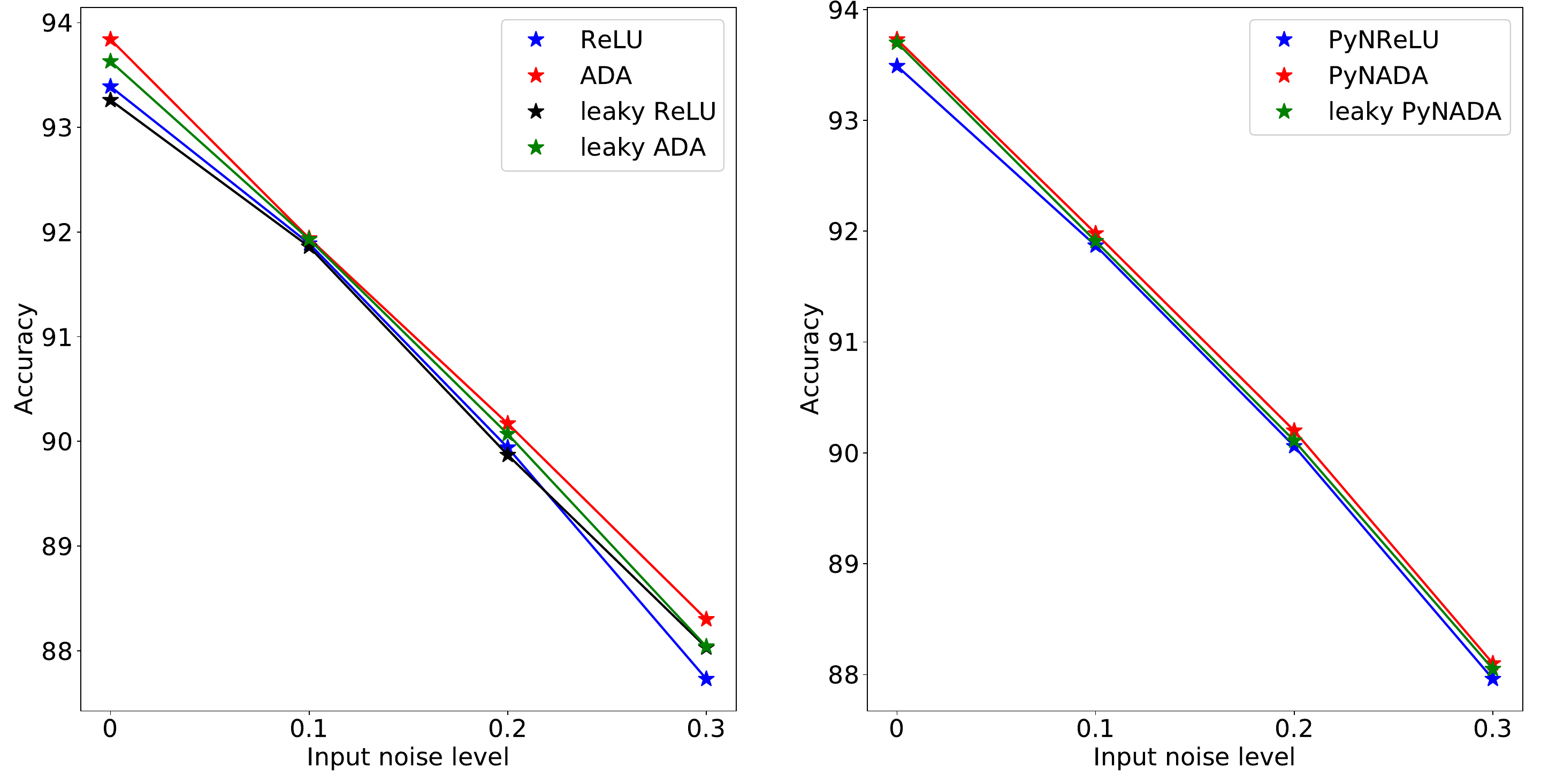}}
\caption{Object class recognition accuracy rates (in \%) for a two-hidden-layer MLP on Fashion-MNIST, with various levels of Gaussian noise applied on the input images. Results are reported with various activations (ReLU, leaky ReLU, ADA, leaky ADA) and artificial neurons (standard, PyNReLU, PyNADA and leaky PyNADA). We add Gaussian noise of $0$ mean and different standard deviations (illustrated on the horizontal axis). Best viewed in color.}\label{noise_input_vgg_fashion_mnist}
\end{center}
\end{figure}

\noindent
{\bf Effect of input noise.}  We study the behavior of ADA and PyNADA when training and testing on noisy images by performing experiments on the Fashion-MNIST data set with two different models, namely MLP-2 and VGG-9. Fashion-MNIST is a clean data set, meaning that the image samples do not contain  noise. To simulate a data set of noisy images, we add Gaussian noise of $0$ mean and different standard deviations, namely $\sigma \in \{0.1, 0.2, 0.3 \}$. We can visualize some randomly chosen images after applying the noise with different standard deviations in Figure \ref{noisy_input_samples_fashionmnist}. We carry out experiments to compare the ADA and leaky ADA functions with ReLU and leaky ReLU, respectively. Likewise, we compare PyNReLU with PyNADA and leaky PyNADA. The corresponding results are presented in Figure \ref{noise_input_mlp_fashion_mnist} and Figure \ref{noise_input_vgg_fashion_mnist}. We observe that the performance of ADA surpasses the performance of ReLU, regardless of the input noise level and the underlying architecture. The same applies to the comparison between PyNReLU and PyNADA. When it comes to classifying noisy input images, we conclude that ADA and PyNADA maintain their competitive edge over ReLU and PyNReLU.

\noindent
{\bf Effect of labeling noise.} To analyze the behavior of ADA and PyNADA when training on noisy labels, we carry out experiments on Fashion-MNIST with two models, MLP-2 and VGG-9. Fashion-MNIST is a well annotated data set. Therefore, we randomly change the labels of a certain fraction ($0.1$, $0.2$ or $0.3$) of the training data set to simulate noisy labeling. The corresponding results are presented in Figure \ref{noise_mlp_fashion_mnist} and Figure \ref{noise_vgg_fashion_mnist}. We perform experiments with ReLU, leaky ReLU, ADA and leaky ADA as activation functions. We also carry out experiments with PyNReLU, PyNADA, and leaky PyNReLU. We observe that ADA and leaky ADA consistently outperform ReLU and leaky ReLU, regardless of the noise level, for both neural networks. The same observation applies when comparing PyNReLU with PyNADA and leaky PyNADA. We conclude that both ADA and PyNADA keep their competitive edge over ReLU and PyNReLU, in the presence of noisy labels.

\begin{figure}  
\begin{center}
\centerline{\includegraphics[width=1.0\columnwidth]{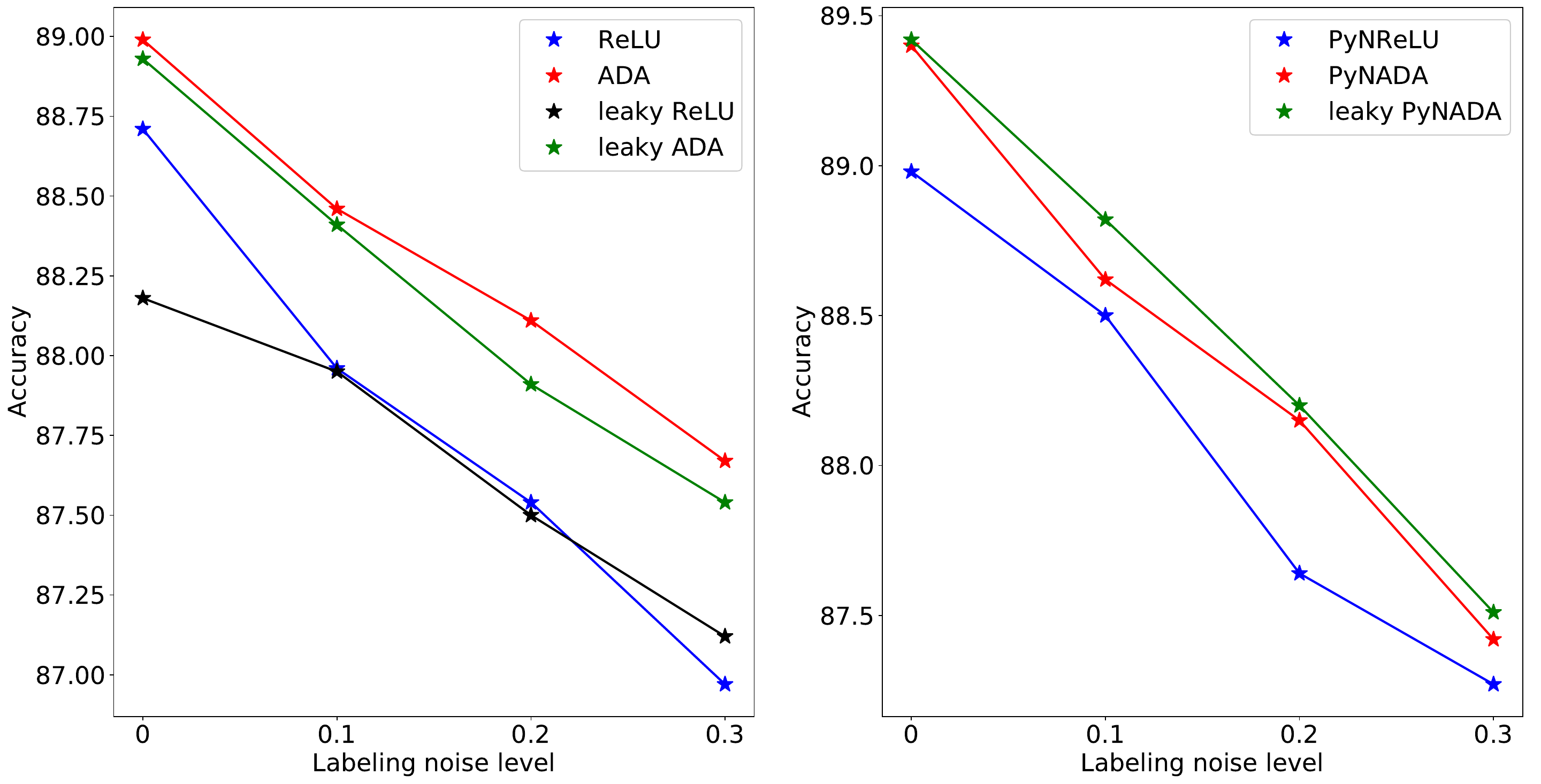}}
\caption{Object class recognition accuracy rates (in \%) for a two-hidden-layer MLP on Fashion-MNIST, with various levels of labeling noise. Results are reported with various activations (ReLU, leaky ReLU, ADA, leaky ADA) and artificial neurons (standard, PyNReLU, PyNADA and leaky PyNADA). We simulate the labeling noise by randomly changing the labels for several ratios (illustrated on the horizontal axis) of the training set. Best viewed in color.}\label{noise_mlp_fashion_mnist} 
\end{center}
\end{figure}

\begin{figure} 
\begin{center}
\centerline{\includegraphics[width=1.0\columnwidth]{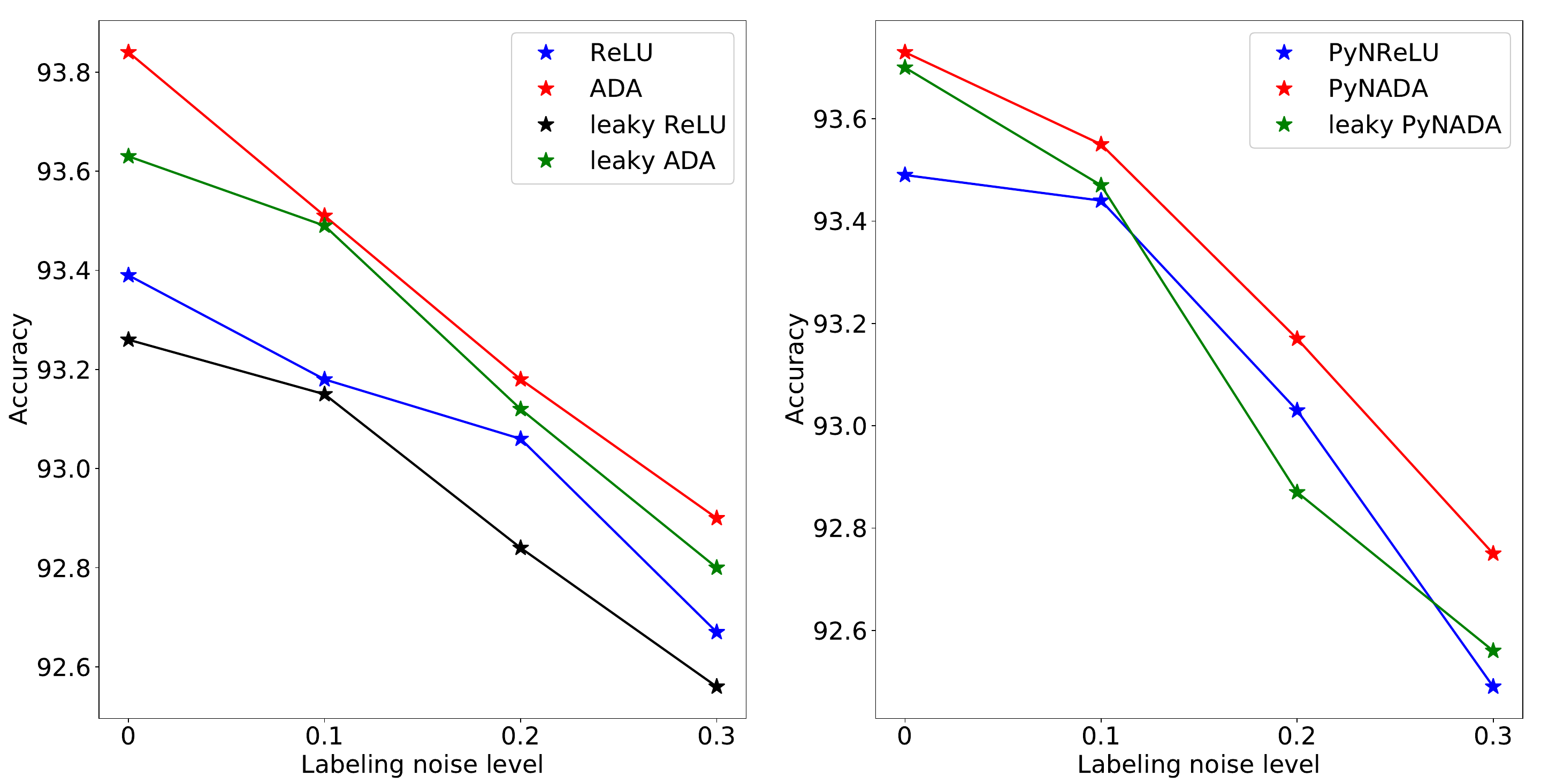}}
\caption{Object class recognition accuracy rates (in \%) for VGG-9 on Fashion-MNIST, with various levels of labeling noise. Results are reported with various activations (ReLU, leaky ReLU, ADA, leaky ADA) and artificial neurons (standard, PyNReLU, PyNADA and leaky PyNADA). We simulate the labeling noise by randomly changing the labels for several ratios (illustrated on the horizontal axis) of the training set. Best viewed in color.}\label{noise_vgg_fashion_mnist}
\end{center}
\end{figure}

\noindent
{\bf Effect of data imbalance.} We investigate the effect of data imbalance when using ADA and PyNADA by performing experiments on Fashion-MNIST with the same models as above, namely MLP-2 and VGG-9. Fashion-MNIST is a balanced data set, meaning that the samples are evenly distributed among classes. To simulate an imbalanced version of Fashion-MNIST, we randomly choose $5$ classes (out of $10$) and keep only $10\%$ of the data samples for the selected classes. We report the \textit{weighted accuracy} for this experiment, to account for the imbalanced nature of the data set. We perform experiments with ReLU, leaky ReLU, ADA and leaky ADA as activation functions, and PyNReLU, PyNADA and leaky PyNADA as neurons, respectively. The corresponding results are presented in Figure \ref{imbalanced_fashion_mnist}. The evaluation shows that ADA and leaky ADA surpass ReLU and leaky ReLU, regardless of the underlying architecture (MLP-2 or VGG-9). Similarly, PyNADA and leaky PyNADA outperform PyNReLU. Overall, the results confirm that the proposed ADA and PyNADA keep their competitive edge over ReLU and PyNReLU, when it comes to solving classification tasks suffering from data imbalance.

\begin{figure} 
\begin{center}
\centerline{\includegraphics[width=1.0\columnwidth]{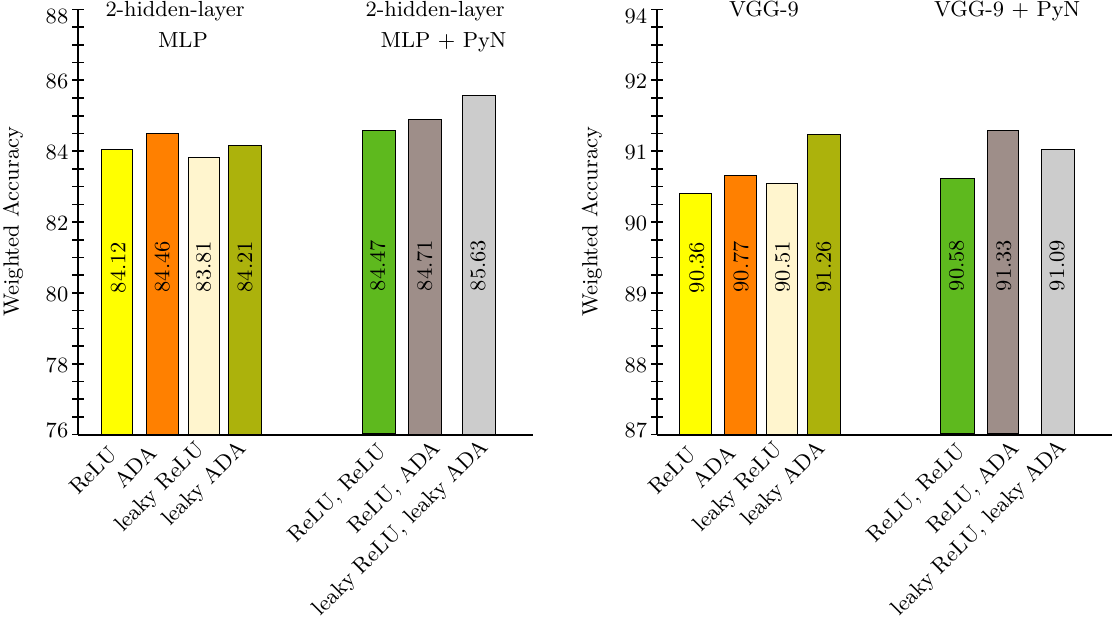}}
\caption{Weighted accuracy rates (in \%) for two neural models (MLP-2 and VGG-9) on a data imbalanced version of Fashion-MNIST. Results are reported with various activations (ReLU, leaky ReLU, ADA, leaky ADA) and artificial neurons (standard, PyNReLU, PyNADA and leaky PyNADA). We simulate the imbalanced version of the data set by randomly dropping $90\%$ of the samples for $5$ out of $10$ classes.}\label{imbalanced_fashion_mnist}
\end{center}
\end{figure}

\subsection{Results on Nonlinear Function Approximation}
\label{sec_results_nfa}

In Figure \ref{fig_sin} and Figure \ref{fig_sinc}, we compare the nonlinear approximation capabilities of our ADA and leaky ADA functions against ReLU, leaky ReLU, RBF and Swish. More precisely, we vary the activation functions of a simple one-hidden-layer MLP with $20$ hidden neurons and one output neuron, and employ the resulting models to approximate two nonlinear functions, $\sin(x)$ and $\sinc(x) = \frac{\sin(x)}{x}$, where $x$ is taken between $[-2\pi, 2\pi]$. When it comes to approximating these functions, ADA and leaky ADA outperform RBF, ReLU and leaky ReLU. Swish is able to output the best approximation for $\sin(x)$, while ADA and leaky ADA are better approximators of the $\sinc(x)$ function. We underline that the chosen neural model is restricted to a basic shallow architecture with only one hidden layer of 20 neurons, thus emphasizing the approximation capabilities of the activation functions rather than the neural architecture itself. Certainly, all compared activation functions can provide much better approximations with deeper and wider architectures, but an empirical study in this direction goes beyond the scope of this work. All in all, our results highlight the intrinsic nonlinear property of the proposed ADA and leaky ADA functions.

\begin{figure}[!t]
\begin{center}
\centerline{\includegraphics[width=1.\columnwidth]{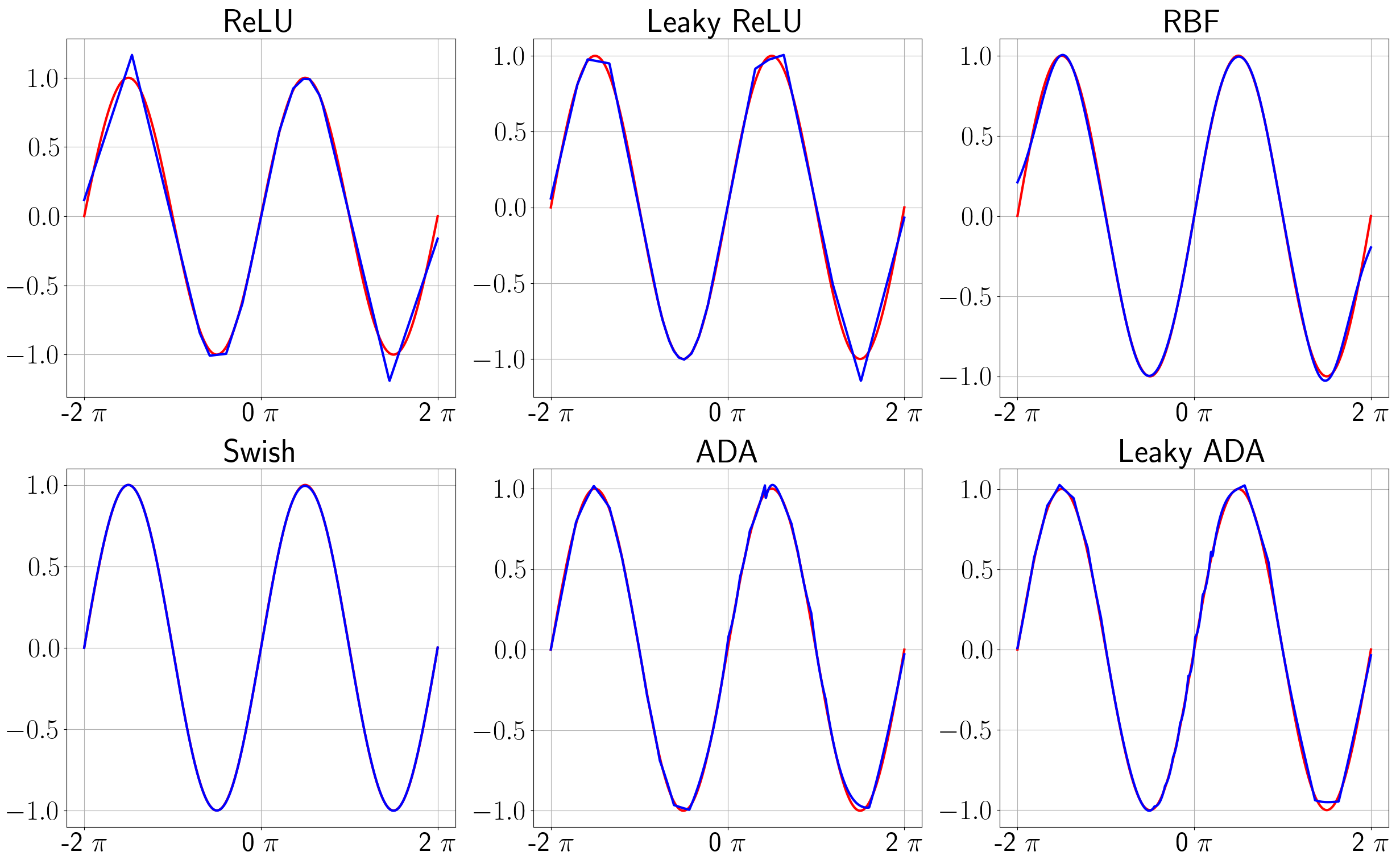}}
\caption{Approximating the $\sin(x)$ function with a one-hidden-layer MLP based on various activation functions. The red graph represents the ground-truth, while the blue graph represents the approximation. Best viewed in color.}
\label{fig_sin}
\end{center}
\end{figure}

\begin{figure}[!t]
\begin{center}
\centerline{\includegraphics[width=1.\columnwidth]{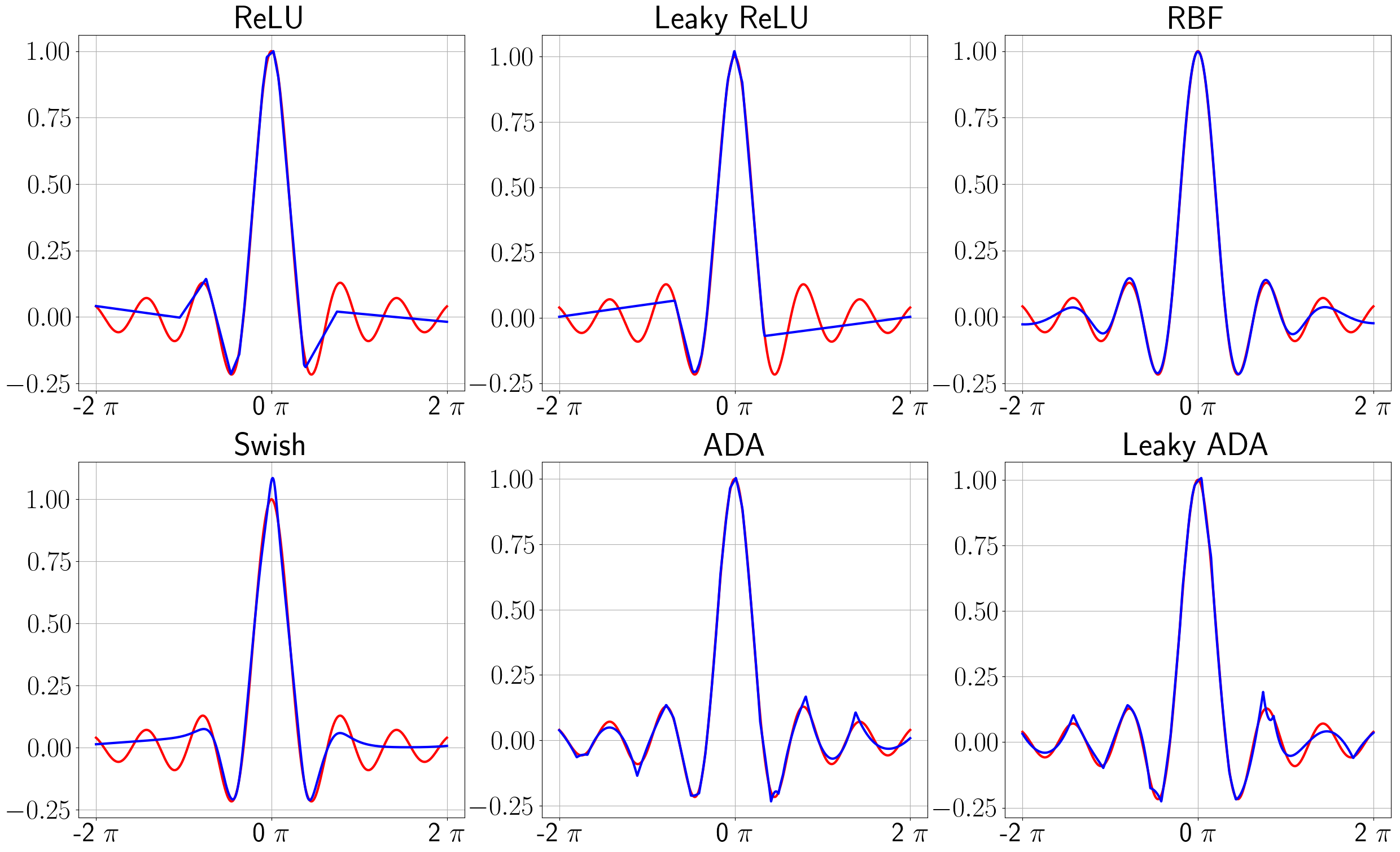}}
\caption{Approximating the $\sinc(x)$ function with a one-hidden-layer MLP based on various activation functions. The red graph represents the ground-truth, while the blue graph represents the approximation. Best viewed in color.}
\label{fig_sinc}
\end{center}
\end{figure}

\subsection{Discussion}
\label{sec_discussion}

We demonstrated the applicability of our activation function (ADA) by integrating it into different CNN (LeNet, VGG, and ResNet) and MLP models and addressing a broad range of tasks: object recognition, age estimation and gender prediction in images; dialect identification in text; emotion recognition in speech. More precisely, we performed experiments on six benchmark data sets from computer vision, signal processing and natural language processing. Our experiments show that we can successfully replace the most used activation function, namely ReLU, with ADA and obtain better performance without requiring any other change to the neural network. In general, ADA can be applied to any neural network without requiring further modifications. In a similar way, we show the benefits of using pyramidal neurons, i.e.~PyNADA. Although we selected a broad range of tasks and domains to carry out our experiments, we strongly believe that ADA and PyNADA are applicable to many other tasks and domains. This advantage stems from the fact that ADA and PyNADA are basic neural building blocks, which can be immediately integrated into most neural network architectures.

By looking at Figure \ref{fig_leaky_adaf3}, we can observe that ADA may lead to vanishing gradients, a potential drawback of the proposed function. However, we never encountered the vanishing gradient problem in practice. We performed many experiments (on six data sets, with multiple architectures) and none of the experiments showed slow convergence. Different from the well-known sigmoid and tanh activation functions, we can control the width of the peak via the parameter $\alpha$, increasing the non-asymptotic region of ADA as necessary to avoid the vanishing gradient problem. Hence, in general, it is sufficient to tune or learn the hyperparameter $\alpha$ of ADA to avoid vanishing gradients.

Another drawback of the proposed activation function is the slower computational speed with respect to ReLU, due to the use of the exponential function. Although our time measurements presented in Table \ref{tab_MOROCO} showed marginal slow downs, we recommend keeping an eye on the computational time when integrating ADA and PyNADA into very large neural models deployed in real-time applications. We underline that this drawback is not particular to our activation function, e.g.~sigmoid, tanh, Swish \cite{Ramachandran-ICLRW-2018} and ELU \cite{Clevert-ICLR-2016} also imply using the exponential function.

\section{Conclusion and Future Work}
\label{sec_conclusion}

In this paper, we proposed a biologically-inspired activation function and a new model of artificial neuron. The novel apical dendrite activation function $(i)$ enables individual artificial neurons to solve nonlinearly separable problems such as the XOR logical function, and $(ii)$ brings significant performance improvements for a broad range of neural architectures and tasks. Indeed, we observed consistent performance improvements over the most popular activation function, ReLU, across six benchmark data sets. The proposed ADA also outperformed other activations, RBF and Swish, that enable individual artificial neurons to solve XOR. Even though RBF and Swish share the capability of solving XOR with ADA, the accuracy rates of RBF and Swish across the six evaluation benchmarks are inconsistent, in some cases even failing to converge. 
The proposed pyramidal neural design represents another way to further boost the performance. Notably, we observed the largest performance improvements when we used ADA instead of ReLU, RBF or Swish for the apical tuft of the pyramidal neurons. In conclusion, we believe that the biologically-inspired ADA and PyNADA are useful additions to the set of deep learning building blocks. 

\noindent {\bf Future work.}
Our research also opens a few directions of future research. Since the activation dampens along the positive side of the domain, we believe it is worth investigating if ADA is more robust to out-of-distribution or adversarial examples. As the gradient saturates on the positive side, other directions of study are to inject noise into ADA to avoid saturation \citep{Gulcehre-ICML-2016} or to employ alternative optimization methods (that do not rely on gradients) in conjunction with ADA. 

\section{Acknowledgments}

The authors thank reviewers for their valuable feedback, which led to significant improvements of the manuscript.

\section{Declarations}

\subsection{Compliance with Ethical Standards}
This research does not involve human participants and/or animals.

\subsection{Funding}
This work was supported by a grant of the Romanian Ministry of Education and Research, CNCS - UEFISCDI, project number PN-III-P1-1.1-TE-2019-0235, within PNCDI III.

\subsection{Conflicts of Interest}
The authors have no competing interests to declare that are relevant to the content of this article.

\subsection{Data Availability}

The data sets used throughout the experiments are publicly available online.



\bibliography{sn-bibliography}


\end{document}